\documentclass[12pt]{article}
 
\usepackage[margin=1in]{geometry}
\usepackage{amsmath,amsthm,amssymb,mathtools}
\usepackage{zymacros}
\usepackage{graphicx}
\usepackage{subcaption}
\usepackage{color}
\usepackage{enumerate}
\usepackage[numbers]{natbib}
\usepackage{bm}        

\DeclareMathOperator*{\EE}{\mathbb{E}}

\begin{document}
 
 
\title{Statistical Learning Guarantees for Group-Invariant Barron Functions}
\author{
  Yahong Yang\thanks{Email: \texttt{yyang3194@gatech.edu}} \quad
  Wei Zhu\thanks{Email: \texttt{weizhu@gatech.edu}}\\
  School of Mathematics, Georgia Institute of Technology, Atlanta, GA, USA
}
 
\maketitle
 \begin{abstract}
We investigate the generalization error of group-invariant neural networks within the Barron framework. Our analysis shows that incorporating group-invariant structures introduces a group-dependent factor $\delta_{G,\Gamma,\sigma} \le 1$ into the approximation rate. When this factor is small, group invariance yields substantial improvements in approximation accuracy. On the estimation side, we establish that the Rademacher complexity of the group-invariant class is no larger than that of the non-invariant counterpart, implying that the estimation error remains unaffected by the incorporation of symmetry. Consequently, the generalization error can improve significantly when learning functions with inherent group symmetries. We further provide illustrative examples demonstrating both favorable cases, where $\delta_{G,\Gamma,\sigma}\approx |G|^{-1}$, and unfavorable ones, where $\delta_{G,\Gamma,\sigma}\approx 1$. Overall, our results offer a rigorous theoretical foundation showing that encoding group-invariant structures in neural networks leads to clear statistical advantages for symmetric target functions.
\end{abstract}

\section{Introduction}
Deep neural networks are effective function approximators across machine learning and scientific computing, from regression~\cite{goodfellow2016deep} to solving partial differential equations~\cite{lagaris1998artificial,weinan2017deep,raissi2019physics,hao2024newton,he2023mgno}. On the theoretical side, extensive work has characterized approximation and generalization for broad classes of targets, such as smooth or Barron-type functions~\cite{hornik1991approximation,barron2002universal,yarotsky2017error,zhou2020universality,lu2021learning,yang2023nearly,yang2025deep}.

A growing literature shows that building known structure into models---such as multiscale behavior \cite{fan2019multiscale,liu2020multi,hao2024multiscale}, low-dimensional geometry \cite{chen2019efficient,poggio2017and,beck2019machine}, or group symmetry (invariance and equivariance)~\cite{chen2023implicit,chen2023sample,chen2025statistical,tahmasebi2023sample}---can substantially improve both approximation and generalization. This work focuses on two-layer networks that enforce invariance under a finite group. Within a Barron-space framework, we ask a concrete question: when the target is group-invariant, how much can one gain statistically by encoding that symmetry in the learner?

Statistical guarantees for learning with invariance have been developed along several fronts. Yarotsky~\cite{yarotsky2022universal} established universal approximation results for invariant maps by neural networks, though without addressing potential improvements in approximation rates. Sokoli\'{c} et al.~\cite{pmlr-v54-sokolic17a} analyze generalization of invariant classifiers under structural assumptions on the domain and group action. Chen et al.~\cite{chen2023sample} study divergence estimation for group-invariant distributions, revealing nuanced sample-complexity reductions from symmetry; this line extends to continuous Lie group actions on compact manifolds in~\cite{tahmasebi2024sample}. Similar gains have been demonstrated for generative models, including GANs~\cite{chen2025statistical} and diffusion models~\cite{chen2024equivariant}. Provable improvements in generalization have also been established for linear predictors~\cite{elesedy2021provably} and kernel methods~\cite{elesedy2021provably_kernel}. Beyond expressivity and sample complexity, symmetry also affects optimization: Chen and Zhu~\cite{chen2023implicit} analyze the implicit bias of multilayer linear group-invariant networks, clarifying how training dynamics interact with symmetry.

In this paper, we rigorously analyze the approximation and estimation error of group-invariant networks within the Barron-space framework \cite{ma2022barron,li2020complexity,ma2019priori}. On the approximation side, we prove that group averaging improves the approximation rate by a multiplicative factor $\delta_{G,\Gamma,\sigma}\le 1$ (Theorem~\ref{app}). Thus, for fixed width $m$, the approximation error decreases in direct proportion to $\delta_{G,\Gamma,\sigma}$. On the estimation side, we bound the empirical Rademacher complexity and show that---without additional assumptions on the data distribution---the group-invariant class does not have higher complexity than its non-invariant counterpart (Proposition~\ref{prop:boundcom}). Combining these results yields a generalization bound (Theorem~\ref{thm:gen}) in which the $G$-dependent improvement enters solely through $\delta_{G,\Gamma,\sigma}$.

Finally, we note that the factor $\delta_{G,\Gamma,\sigma}$ depends in a nuanced way on the group action, the activation function $\sigma$, and the admissible measure family $\Gamma$. We illustrate this with examples showing when $\delta_{G,\Gamma,\sigma}\approx |G|^{-1}$ (substantial gains) and when $\delta_{G,\Gamma,\sigma}\approx 1$ little or no gain), underscoring that the benefits of symmetry are case dependent.

\section{Statistical learning of group-invariant Barron functions}\label{sec:app}

\subsection{Background and motivation}
\label{sec:background}
We begin with a brief review of group actions and group invariance.

\begin{definition}[Group and group actions]\label{def:group_representation}
A set $G$ equipped with a binary operation is called a group if the operation satisfies the axioms of associativity, identity, and invertibility. Let $\Omega$ be a domain. $T:G\times \Omega\to \Omega$ is a group action on $\Omega$ if
\begin{align}
    T_e(\vx) = \vx, \quad \text{and}\quad T_{g_2}\circ T_{g_1}(\vx)=T_{g_2\cdot g_1}(\vx), \quad \forall g_1, g_2\in G, ~\vx\in \Omega,
\end{align}
where $e\in G$ is the identity element. When the group action $T$ is clear from the context, we also abbreviate the group action
$T_g\vx$ simply as $g\vx$.
\end{definition}

In this paper, we focus on \emph{linear group actions}. Specifically, we assume that the domain $\Omega \subset V$, where $V$ is a vector space, and define
\begin{align}
    T_g(\vx)=\rho(g)(\vx)
\end{align}
where $\rho(g) \in \mathrm{GL}(V)$ is an invertible linear transformation on $V$. In other words, $\rho : G \to \mathrm{GL}(V)$ is a group homomorphism. With this definition in place, we introduce the notion of \emph{group-invariant} functions.

\begin{definition}[\(G\)-invariant function]
\label{def:inv}
A function \(f : \Omega \to \mathbb{R}\) is called \emph{\(G\)-invariant} if
\begin{align}
\label{eq:g-inv-def}
   f(g\vx) = f(\vx), \qquad \forall\,g\in G,\;\vx\in \Omega. 
\end{align}
\end{definition}

A more general notion is that of a \emph{\(G\)-equivariant} function. Let $T^{(1)}: G \times \Omega_1 \to \Omega_1$ and $T^{(2)}: G \times \Omega_2 \to \Omega_2$ be two group actions on $\Omega_1$ and $\Omega_2$, respectively. For a mapping $f: \Omega_1 \to \Omega_2$, we say that $f$ is $G$-equivariant if
\begin{align}
    f(T_g^{(1)} \vx) = T_g^{(2)}f(\vx), \quad \forall g\in G, ~\vx\in\Omega_1
\end{align}
In particular, when $T_g^{(2)} \equiv \mathrm{Id}$, this reduces to the definition of a \(G\)-invariant function~\eqref{eq:g-inv-def}. This definition also includes anti-symmetric functions, as discussed in \cite{abrahamsen2025anti}.
In this work, we restrict attention to the special case of \(G\)-invariant functions and leave the broader \(G\)-equivariant setting for future study.

Throughout the paper, we consider $\Omega \subset \mathbb{R}^d$ to be a bounded domain endowed with a linear action of a finite group $G$, denoted by $T_g(\vx)=g\vx$.
\begin{assumption}
\label{assump:bounded_omega}
    Without loss of generality, assume that $\Omega$ is contained in the unit $\ell_\infty$ ball, i.e.,
    \[
        \max_{\vx \in \Omega} \|\vx\|_{\infty} \leq 1.
    \]
\end{assumption}
Our goal is to establish statistical learning guarantees for a $G$-invariant target function $f_*$, given finitely many noisy observations on the domain $\Omega$. We assume that the target function $f_*$ lies in a \textit{generalized Barron space}~\cite{ma2022barron}, defined as follows.

\begin{definition}[$\sigma$-activated $\Gamma$-Barron space]
Let \(d\in\mathbb{N}\) and \(\Omega\subset\mathbb{R}^{d}\).
Consider an activation function $\sigma:\mathbb{R}\to\mathbb{R}$ satisfying the mild condition to be detailed in Assumption~\ref{assump:act}, and a subset $\Gamma\subset \mathcal{P}(\mathbb{R}^{d+2})$ of probability measures on $\mathbb{R}^{d+2}$.

A function \(f:\Omega\to\mathbb{R}\) is said to belong to the \emph{($\sigma$-activated) $\Gamma$-Barron space}, denoted \(\mathcal{B}_{\Gamma}(\Omega)\),  
if there exists a probability measure \(\rho\in\Gamma\) such that
\[
f(\vx) = f_{\rho}(\vx)
:= \mathbb{E}_{(a,\vw,b)\sim\rho}
       \bigl[ a\,\sigma(\vw\!\cdot\!\vx + b) \bigr],
\qquad
\forall\,\vx\in\Omega,
\]
and its \emph{$\Gamma$-Barron norm} is finite:
\begin{align}
\label{eq:def_barron_norm}
    \|f\|_{\mathcal{B}_{\Gamma}}
:= \inf_{\rho\in\Gamma:\,f=f_{\rho}}
  \sqrt{
    \mathbb{E}_{(a,\vw,b)\sim\rho}
    \,|a|^2\,\bigl(\|\vw\|_{1} + |b| + 1\bigr)^2
  } < \infty.
\end{align}
Furthermore, we denote $\mathcal{B}_{\Gamma}^G$ the subset of $G$-invariant $\Gamma$-Barron functions, i.e.,
\begin{align}
    \mathcal{B}_{\Gamma}^G(\Omega) \coloneqq \left\{f\in \mathcal{B}_{\Gamma}:~f(g\vx) = f(\vx), ~\forall g\in G, ~\forall \vx\in \Omega\right\}
\end{align}
\end{definition}

When $\Gamma$ is taken to be the set of all probability measures on $\mathbb{R}^{d+2}$, our definition of the $\Gamma$-Barron norm reduces to the standard Barron norm introduced in \cite{ma2022barron} for general activation functions $\sigma$. 
In our formulation, the additive constant \(1\) in \(\|\vw\|_1 + |b| + 1\) is included deliberately in the definition of the $\Gamma$-Barron norm~\eqref{eq:def_barron_norm}. This adjustment simplifies the subsequent generalization error analysis by ensuring uniform bounds on both the approximation error and the Rademacher complexity of the neural network class, in line with \cite{li2020complexity,ma2022barron}.  
 It is also worth noting that, for $\sigma = \operatorname{ReLU}$, our norm is equivalent to the standard Barron norm without the $+1$ term, thanks to the homogeneity of the ReLU function and the infimum in the definition.

We next state a mild assumption on the activation functions $\sigma$ appearing in the above definition of Barron spaces.

\begin{assumption}\label{assump:act}
    The activation function $\sigma:\mathbb{R}\to\mathbb{R}$ is a Lipschitz continuous function with Lipschitz constant \(L_\sigma\).  
    Moreover,  there exists a constant $\gamma(\sigma) > 0$ such that for any $\epsilon > 0$, $\sigma$ can be uniformly approximated on $\mathbb{R}$ by a ReLU neural network of the form
    \[
    \sum_{k=1}^K \bar{a}_k \operatorname{ReLU}\left(\bar{w}_k x + \bar{b}_k\right) \coloneqq
    \sum_{k=1}^K \bar{a}_k \max\left(\bar{w}_k x + \bar{b}_k, 0\right)
    \]
    satisfying
    \[
    \sup_{x\in\mathbb{R}} \left| \sigma(x) - \sum_{k=1}^K \bar{a}_k \operatorname{ReLU}(\bar{w}_k x + \bar{b}_k) \right| \le \epsilon,
    \quad
    \sum_{k=1}^K |\bar{a}_k| \big(|\bar{w}_k| + |\bar{b}_k|\big) \le \gamma(\sigma).
    \]
\end{assumption}
\begin{remark}\label{rmk: act}
Many commonly used activation functions satisfy Assumption~\ref{assump:act}. Since Lipschitz continuity is straightforward to verify for the examples below, we omit the details for brevity.
\begin{enumerate}
    \item \textbf{\em ReLU}: The standard ReLU function trivially satisfies the property with $\gamma(\operatorname{ReLU})=1$.
    \item \textbf{\em Piecewise linear activations}: Any piecewise linear activation can be represented exactly by a ReLU network. Examples include the hat functions, LeakyReLU, HardSigmoid, and HardTanh, defined as
    \begin{align*}
    \operatorname{LeakyReLU}(x) &= \lambda x + (1-\lambda)\operatorname{ReLU}(x), \quad 0 < \lambda < 1,\\    
    {\rm HardSigmoid}(x) &= \max\!\left(0, \min\!\left(1, \frac{x+1}{2}\right)\right), \\
    {\rm HardTanh}(x) &= \max(-1, \min(1, x)).
    \end{align*}
    \item \textbf{\em Functions satisfying the condition of \cite[Theorem~1]{li2020complexity}}:  Let $\sigma:\mathbb{R}\to\mathbb{R}$ be continuous, twice weakly differentiable, and assume $\sigma''$ is locally Riemann integrable. If
    \begin{align}
    \label{eq:gamm_0_sigma}
        \gamma_0(\sigma) \coloneqq \int_{\mathbb{R}} |\sigma''(x)|(|x|+1)\,\D x < +\infty,
    \end{align}
    then Assumption~\ref{assump:act} holds.  Condition~\eqref{eq:gamm_0_sigma} encompasses a broad class of widely used activations, such as
    \begin{itemize}
        \item \textbf{\em Non-smooth ReLU-shaped activations}. Examples include:
        \begin{align}
        \label{eq:celu_selu}
            \mathrm{CELU}(x) & =
            \begin{cases}
            x, & x \ge 0, \\
            \beta(e^{x/\beta} - 1), & x < 0,
            \end{cases}, ~~
            \mathrm{SELU}(x) & =
            \begin{cases}
            \lambda x, & x \ge 0, \\
            \lambda \alpha(e^x - 1), & x < 0,
            \end{cases}            
        \end{align}
        where $\alpha \in \mathbb{R}$, $\beta, \lambda > 0$.  When $\alpha = 1$, these functions are twice weakly differentiable, with
        \[
        |\sigma''(x)| \le \frac{C}{e^x} \ \text{as } x \to -\infty,
        \quad
        \sigma''(x) = 0 \ \text{for } x > 0,
        \]
        which ensures that $\gamma_0(\sigma)$ is finite.
        \item \textbf{\em Smooth Heaviside-type activations}. These are smooth approximations of the step function, such as:
        \[
        \sigma(x) = {\rm Sigmoid}(x), \quad {\rm Tanh}({\rm Softplus}(x)), \quad {\rm erf}(x).
        \]
        Such functions are in $C^\infty(\mathbb{R})$ and all derivatives $\sigma^{(k)}(x)$ decay exponentially as $x \to \pm\infty$ for any finite $k$.  
        Hence, $\gamma_0(\sigma) < \infty$ and they satisfy Assumption~\ref{assump:act} via \cite[Theorem~1]{li2020complexity}.
        \item \textbf{\em Smooth ReLU-shaped activations}. For any $\sigma$ in the previous category, the function
        \[
        x \cdot \sigma(x)
        \]
        also satisfies condition~\eqref{eq:gamm_0_sigma}.
        Examples include:
        \[
        \mathrm{SiLU}(x) = x \cdot {\rm Sigmoid}(x),~
        \mathrm{Mish}(x) = x \cdot {\rm Tanh}({\rm Softplus}(x)),~
        \mathrm{GELU}(x) = x \cdot {\rm erf}(x)
        \]
        These are functions in $C^\infty(\mathbb{R})$ and have normalized forms $\widetilde{\sigma}(x) = \sigma(x)/x$ equal to the Heaviside-type functions in the previous category, whose derivatives decay exponentially, ensuring $\gamma_0(\sigma) < \infty$.
        \item \textbf{\em Compactly supported activations}. If $\sigma$ has continuous second derivatives and compact support, then by the Extreme Value Theorem,  
        $\gamma_0(\sigma) < \infty$.
    \end{itemize}
\end{enumerate}
More activations satisfying Assumption~\ref{assump:act} are discussed in detail in \cite{li2020complexity,yang2025deep}.
\end{remark}

Functions in the $\Gamma$-Barron space can be viewed as infinite-width, two-layer $\sigma$-activated networks whose parameters $(a,\vw,b)$ are drawn from a probability measure in $\Gamma \subset \mathcal{P}(\mathbb{R}^{d+2})$. A substantial body of literature studies the generalization theory of learning Barron functions using finite-width two-layer networks of the form
\begin{align}
    f_m(\vx, \Theta) = \frac{1}{m}\sum_{i=1}^m a_i\sigma\left(\vw_i\cdot \vx + b_i\right),
\end{align}
where $\Theta=(a_i, \vw_i, b_i)_{i=1}^m$ is the collection of all network parameters;
see, e.g.,~\cite{ma2022barron,ma2019priori,li2020complexity,siegel2024sharp,siegel2020approximation}.

In our setting, the target function $f_*\in\mathcal{B}_\Gamma^G(\Omega)$ is assumed to be $G$-invariant (Definition~\ref{def:inv}), motivating the use of $G$-invariant networks obtained by group-averaging:
\begin{align}
    f_m^G(\vx;\Theta) = \mathcal{G}f_m(\vx;\Theta)= \frac{1}{m|G|}\sum_{i=1}^m a_i\sum_{g\in G}\sigma\left(\vw_i\cdot g\vx + b_i\right),
\end{align}
where
\begin{align}
    \mathcal{G}f(\vx) \coloneqq  \frac{1}{|G|} \sum_{g\in G} f(g\vx).
\end{align}

Generalization error for Barron spaces has been extensively studied in both under- and over-parameterized regimes. The under-parameterized case is addressed in \cite{lu2021priori,xu2024refined}, while the over-parameterized regime has received greater attention due to exponential decay in the early stage of training explained by neural tangent kernel theory \cite{allen2019convergence,arora2019exact,du2018gradient,yang2025homotopy,chen2023phase}. In this setting, works such as \cite{ma2022barron,ma2019priori,li2020complexity,yang2024nonparametric,liu2024learning} establish and refine convergence rates: $\mathcal{O}(M^{-1/2})$ in \cite{ma2022barron,ma2019priori,li2020complexity}, improved to $M^{-(d+2)/(2d+2)}$ in \cite{liu2024learning}, and further to $M^{-(d+3)/(2d+3)}$ in \cite{yang2024nonparametric} using local Rademacher complexity and spherical harmonics. These refinements focus on ReLU activation, while sharper rates for more general activations remain open.

In the following subsections, we adopt the framework of \cite{ma2022barron,ma2019priori,li2020complexity} to rigorously analyze the approximation error, the estimation error, and the resulting generalization error for learning $G$-invariant $\sigma$-activated Barron functions. Our goal is to quantify how much incorporating group-invariant structure can improve learning performance compared with standard (non-invariant) networks.

\subsection{Approximation error}
\label{sec:approx_error}





Let $\mu$ be any probability measure on the domain $\Omega$. We now present our main approximation theorem for $G$-invariant networks in approximating $G$-invariant Barron functions $f_* \in \mathcal{B}_\Gamma^G(\Omega)$, stated with respect to the $L^2(\Omega,\D\mu) = L^2(\D\mu)$ norm:

\begin{theorem}\label{app}
Let $\sigma$ be a Lipschitz continuous function with Lipschitz constant \(L_\sigma\), and let \(f_*\in\mathcal{B}_\Gamma^G(\Omega)\) be a \(G\)-invariant function belonging to the $\sigma$-activated \(\Gamma\)-Barron space. Then there exists a shallow neural network of the form
\begin{align*}
f_m^G(\vx, \Theta) = \frac{1}{m|G|}\sum_{i=1}^{m}
         a_i \sum_{g\in G} \sigma\!\bigl(\vw_i\cdot g\vx + b_i\bigr)
         = \frac{1}{m|G|}\sum_{i=1}^{m}
         a_i \sum_{g\in G} \sigma\!\bigl((g^\top\vw_i)\cdot \vx + b_i\bigr)
\end{align*}
where $\Theta = (a_i, \vw_i, b_i)_{i=1}^m$,  such that
\begin{align}
\label{eq:G_approx_m}
\left\{
\begin{aligned}
\|f_* - f_m^G(\cdot; \Theta)\|_{L^2(\D \mu)}^{2}
\le \frac{3 (L_\sigma+|\sigma(0)|)^2 \delta_{G,\Gamma,\sigma} \|f_*\|_{\fB_\Gamma}^{2}}{m},
\\
\frac{1}{m}\sum_{i=1}^{m}|a_i|^2\,(\|\vw_i\|_{1}+|b_i|+1)^2
\le 2\|f_*\|^2_{\fB_\Gamma},
\end{aligned}\right.
\end{align}
where
\[
\delta_{G,\Gamma,\sigma}:=\min\{1,\delta_{G,\Gamma,\sigma}^*\},\quad\delta_{G,\Gamma,\sigma}^*
= \sup_{\rho\in\Gamma,\ \vx\in\Omega}
\frac{
\mathbb{E}_{(a,\vw,b)\sim\rho}
\left[\left(a\,\fG\,\sigma(\vw\cdot\vx + b)\right)^2\right]
}{
\mathbb{E}_{(a,\vw,b)\sim\rho}
\left[\left(a\,\sigma(\vw\cdot\vx + b)\right)^2\right]
}.
\]

\end{theorem}


\begin{proof}
The proof is inspired by~\cite{ma2022barron}. By the definition of Barron norm, for any $\varepsilon\in (0, \frac{1}{5})$, we choose a probability measure
\(\rho\in\Gamma\) such that
\[
f_*(\vx)=
\mathbb{E}_{(a,\vw,b)\sim\rho}
  \bigl[a\,\sigma(\vw\!\cdot\!\vx+b)\bigr],
\quad
\mathbb{E}_{(a,\vw,b)\sim\rho}
  |a|^2(\|\vw\|_{1}+|b|+1)^2
  \le(1+\varepsilon)\,\|f_*\|^2_{\fB_\Gamma}.
\]
Because \(\fG f = f\) for every \(G\)-invariant \(f\), applying \(\fG\) gives
\[
f_*(\vx)=
\frac{1}{|G|}\,
\mathbb{E}_{(a,\vw,b)\sim\rho}
  \Bigl[
     a\sum_{g\in G}
       \sigma(\vw\cdot g\vx+b)
  \Bigr].
\]
For \(\vtheta=(a,\vw,b)\sim\rho\) define
\begin{equation*}
\phi(\vx;\vtheta):=
\frac{a}{|G|}\sum_{g\in G}\sigma(\vw\cdot g\vx+b).
\end{equation*}
Then
\begin{align*}
    \mathbb{E}_{\vtheta\sim\rho}\phi(\vx;\vtheta)=f_*(\vx).    
\end{align*}
Let \(\Theta=\{\vtheta_i\}_{i=1}^{m}\) be i.i.d.\ samples from \(\rho\) and set
\[
f_m^G(\vx;\Theta):=
\frac{1}{m}\sum_{i=1}^{m}\phi(\vx;\vtheta_i),
\qquad
\fE(\Theta):=
\bigl\|
   \,f_* - f_m^G(\,\cdot\,;\Theta)
\bigr\|_{L^{2}(\D\mu)}^{2}.
\]
The mean squared error is thus
\begin{align}
\nonumber
\mathbb{E}_\Theta \fE(\Theta)
&=
\mathbb{E}_\Theta
\int_{\Omega}
   \bigl|
       f_*(\vx) - f_m^G(\vx;\Theta)
   \bigr|^{2}\,\D \mu (\vx)         \\
   \nonumber
&=
\int_{\Omega}
\mathbb{E}_\Theta
   \left|
       f_*(\vx) - {\tfrac1m}\!\sum_{i=1}^{m}\phi(\vx;\vtheta_i)
   \right|^{2}\,\D \mu  (\vx)        \\
   \nonumber
&=
\frac{1}{m^2}\!
\int_{\Omega}
\mathbb{E}_\Theta
   \sum_{i,j=1}^{m}
     \bigl(f_*(\vx)-\phi(\vx;\vtheta_i)\bigr)
     \bigl(f_*(\vx)-\phi(\vx;\vtheta_j)\bigr)\,\D \mu(\vx)\\
     \label{eq:mse_bound_E}
&\le
\frac{1}{m}
\int_{\Omega}
  \mathbb{E}_{\vtheta\sim\rho}
    \phi(\vx;\vtheta)^{2}\,\D \mu(\vx),
\end{align}
where the last inequality uses independence of $\vtheta_i$ and 
\(\mathbb{E}_{\vtheta}\phi(\vx;\vtheta)=f_*(\vx)\).

By the definition of $\delta_{G,\Gamma,\sigma}$ we have, for all $\vx\in\Omega$,
\begin{align*}
\mathbb{E}_{\vtheta\sim\rho}\phi(\vx;\vtheta)^{2}
\;\le\; 
\delta_{G,\Gamma,\sigma}^* \,
\mathbb{E}_{(a,\vw,b)\sim\rho}
\left[\left(a\,\sigma(\vw\cdot\vx + b)\right)^2\right].
\end{align*}
Furthermore, observe that
\begin{align*}
   \bigl|a\,\sigma(\vw\cdot\vx + b)\bigr|
   &= \bigl|a\sigma(\vw\cdot \vx + b)\bigr| \\
   &\le |a|\,\bigl|\sigma(\vw\cdot \vx + b)-\sigma(0)\bigr| + |a\sigma(0)| \\
   &\le L_\sigma|a|\,|\vw\cdot \vx + b| + |a\sigma(0)| \\
   &\le (L_\sigma|a|)(\|\vw\|_1+|b|) + |a\sigma(0)| \\
   &\le (L_\sigma+|\sigma(0)|)\,|a|\,(\|\vw\|_1+|b|+1),
\end{align*}
where, in the second-to-last inequality, we use the assumption $\|\vx\|_{\infty} \le 1$ for all $\vx \in \Omega$. Hence
\begin{equation}
\mathbb{E}_{\vtheta\sim\rho}\phi(\vx;\vtheta)^{2}
\;\le\;
(1+\varepsilon)\,\|f_*\|_{\fB_\Gamma}^{2}\,
(L_\sigma+|\sigma(0)|)^2\,\delta_{G,\Gamma,\sigma}^*,
\qquad \forall \vx\in\Omega.
\label{eq:com1}
\end{equation}Similarly,
\begin{align*}
   \bigl|a\,\fG\,\sigma(\vw\cdot\vx + b)\bigr|
   &= \frac{1}{|G|}\sum_{g\in G}\bigl|a\sigma(\vw\cdot g\vx + b)\bigr| \\
   &\le \frac{1}{|G|}\sum_{g\in G}\Bigl(L_\sigma|a|\,|\vw\cdot g\vx + b| + |a\sigma(0)|\Bigr) \\
   &\le (L_\sigma+|\sigma(0)|)\,|a|\,(\|\vw\|_1+|b|+1).
\end{align*}
Thus we also have
\begin{equation}
\mathbb{E}_{\vtheta\sim\rho}\phi(\vx;\vtheta)^{2}
\;\le\;
(1+\varepsilon)\,\|f_*\|_{\fB_\Gamma}^{2}\,(L_\sigma+|\sigma(0)|)^2,
\qquad \forall \vx\in\Omega.
\label{eq:com2}
\end{equation}Finally, since $\delta_{G,\Gamma,\sigma}=\min\{1,\delta_{G,\Gamma,\sigma}^*\}$, combining
\eqref{eq:com1} and \eqref{eq:com2} yields
\begin{equation}
\label{eq:mse_bound_phi}
\mathbb{E}_{\vtheta\sim\rho}\phi(\vx;\vtheta)^{2}
\;\le\;
(1+\varepsilon)\,\|f_*\|_{\fB_\Gamma}^{2}\,(L_\sigma+|\sigma(0)|)^2\,\delta_{G,\Gamma,\sigma},
\qquad \forall \vx\in\Omega.
\end{equation}
Since $\mu(\Omega)=1$ is a probability measure on $\Omega$, substituting \eqref{eq:mse_bound_phi} into Eq.~\eqref{eq:mse_bound_E} yields 
\begin{align*}
    \mathbb{E}_\Theta \fE(\Theta)
\le
\frac{(1+\varepsilon)\,\|f_*\|_{\fB_\Gamma}^{2}\,(L_\sigma+|\sigma(0)|)^2\delta_{G,\Gamma,\sigma}}{m},
\end{align*}
with
\begin{align*}
    \mathbb{E}_\Theta
  \Bigl[
     \tfrac1m\!\sum_{i=1}^{m}
       |a_i|^2(\|\vw_i\|_{1}+|b_i|+1)^2
  \Bigr]
\le
(1+\varepsilon)\,\|f_*\|^2_{\fB_\Gamma}.    
\end{align*}
Define the events
\begin{align*}
   & E_1:=
\Bigl\{
  \Theta :
  \fE(\Theta)\le\tfrac{3\,\|f_*\|_{\fB_\Gamma}^{2}\delta_{G,\Gamma,\sigma}(L_\sigma+|\sigma(0)|)^2}{m}
\Bigr\},\\
& E_2:=
\Bigl\{
  \Theta :
  \tfrac1m\!\sum_{i=1}^{m}|a_i|^2(\|\vw_i\|_{1}+|b_i|+1)^2
  \le 2\,\|f_*\|_{\fB_\Gamma}
\Bigr\}.    
\end{align*}
By Markov’s inequality,
\[
\mathbb{P}(E_1)
   \ge 1-\frac{1+\varepsilon}{3},
\qquad
\mathbb{P}(E_2)
   \ge 1-\frac{1+\varepsilon}{2}.
\]
Since \(0<\varepsilon<\tfrac15\), both probabilities are positive and
\(\mathbb{P}(E_1\cap E_2)>0\); thus there exists a realization
\(\Theta\) satisfying both bounds.
\end{proof}

Compared with the traditional Barron framework \cite{ma2019priori,ma2022barron,li2020complexity}, our result for group-invariant Barron functions includes an additional group-dependent factor $\delta_{G,\Gamma,\sigma} \le 1$. When $G=\{e\}$ is the trivial group, $\delta_{G,\Gamma,\sigma}=1$, recovering the classical Barron result. For clarity, we present the classical results in Corollary~\ref{appnoG}, enabling a direct comparison between networks with and without $G$-invariant structure.





\begin{corollary}[\cite{li2020complexity,ma2022barron}]\label{appnoG}
Let $\sigma$ be a Lipschitz continuous function with Lipschitz constant \(L_\sigma\), and let \(f_*\in\mathcal{B}_\Gamma(\Omega)\) be a function belonging to the $\sigma$-activated \(\Gamma\)-Barron space. Then there exists a shallow neural network of the form
\[
f_m(\vx, \Theta) = \frac{1}{m}\sum_{i=1}^{m}
         a_i\, \sigma\!\bigl(\vw_i\cdot\vx + b_i\bigr)
\]
where $\Theta=(a_i, \vw_i, b_i)_{i=1}^m$, such that
\[
\|f_* - f_m\|_{L^2(\D\mu)}^{2}
\le \frac{3 \,(L_\sigma+|\sigma(0)|)^2 \, \|f_*\|_{\fB_\Gamma}^{2}}{m},
\quad
\frac{1}{m}\sum_{i=1}^{m} |a_i|^2 \, \bigl(\|\vw_i\|_{1}+|b_i|+1\bigr)^2
\le 2\,\|f_*\|^2_{\fB_\Gamma}.
\]
\end{corollary}


Comparing Theorem~\ref{app} with Corollary~\ref{appnoG}, we see that if group-invariant structures are not incorporated, neural networks can still approximate the group-invariant target function $f_*\in \mathcal{B}_\Gamma^G(\Omega)$, but the factor $\delta_{G,\Gamma,\sigma} \le 1$ no longer appears. In many cases, as we will discuss later in Section~\ref{sec: example}, $\delta_{G,\Gamma,\sigma}$ can be much smaller than $1$, leading to sharper bounds. This indicates that incorporating structural information about the target function into the neural network can potentially provide substantial benefits.

In the next subsection, we will analyze the \textit{estimation error} of using a $G$-invariant shallow network to learn invariant Barron functions. Combined with Theorem~\ref{app}, this will provide a complete picture of the statistical learning guarantees for Barron functions under group symmetry. Before turning to this analysis, we make a brief remark on another related class of Barron functions.


\begin{remark}
Another related class of spaces is the \emph{spectral Barron space} \cite{barron2002universal,lu2025interpolation,chen2023regularity,siegel2020approximation,siegel2024sharp,liao2025spectral}. These spaces are defined by first selecting a dictionary, typically the Fourier basis, and then considering target functions in the closure of the convex hull of that dictionary. The approximation rate can be improved depending on the regularity of the chosen dictionary; see \cite{siegel2024sharp} for details.

A natural way to incorporate group-invariant structure into this framework is to restrict to a subdictionary of group-invariant functions. However, this makes direct comparison between networks with and without group-invariant structure challenging, since the architecture depends heavily on the dictionary choice. We leave a more detailed investigation of this approach as future work.
\end{remark}

\subsection{Estimation error}
\label{sec:sample_error}

In this subsection, we extend the framework of \cite{ma2022barron,ma2019priori,li2020complexity} to bound the estimation error of $G$-invariant shallow networks for learning group-invariant Barron functions. The key step is to estimate the \textit{Rademacher complexity} of these networks.

\begin{definition}[Rademacher complexity \cite{anthony1999neural}]\label{defrad}
Let $S=\{\vx_1,\vx_2,\ldots,\vx_M\}$ be a set of samples on a domain $\Omega$, and let $\fF$ be a class of real-valued functions defined on $\Omega$. The empirical Rademacher complexity of $\fF$ on $S$ is defined as
\[
\hat{\tR}_S(\fF):=\frac{1}{M}\,
\mathbb{E}_{\Xi_M}\Biggl[\sup_{f\in\fF}\sum_{i=1}^M\xi_i\,f(\vx_i)\Biggr],
\]
where $\Xi_M=\{\xi_1,\xi_2,\ldots,\xi_M\}$ is a set of $M$ independent Rademacher random variables, i.e.,
\(\mathbb{P}(\xi_i=+1)=\mathbb{P}(\xi_i=-1)=\frac{1}{2}\) for $i=1,2,\ldots,M.$
\end{definition}

Before we provide the upper bound for the
Rademacher complexity of two-layer networks, we first need the following two lemmas:
\begin{lemma}[{\cite[Lemma~26.11]{shalev2014understanding}}]\label{lem:B1}
Let $S = (\vx_1,\ldots,\vx_M)$ be $M$ vectors in $\mathbb{R}^d$.  
Then the Rademacher complexity of 
\[
\fH = \{\, \vx \mapsto \vu \cdot \vx \ \mid\ \|\vu\|_{1} \le 1 \,\}
\]
has the following upper bound:
\[
\hat{\tR}_S (\fH)
\ \le\ 
\max_{1 \le i \le M} \|\vx_i\|_{\infty} 
\sqrt{\frac{2\log(2d)}{M}}.
\]
\end{lemma}

\begin{lemma}[{\cite[Corollary~4]{maurer2016vector}}]\label{lem:vector_contraction}
Let $S = (\vx_1,\ldots,\vx_M)$ be $M$ vectors in $\mathbb{R}^d$, let $\mathcal{F}$ be a class of functions
$f : \mathbb{R}^d \to \mathbb{R}^K$, and let $h_i : \mathbb{R}^K \to \mathbb{R}$ be $L$-Lipschitz with respect to the $\ell_2$-norm.  
Then
\[
\mathbb{E}_{\Xi_M} \sup_{f \in \mathcal{F}} \sum_{i=1}^M \xi_i\, h_i\big( f(\vx_i) \big)
\ \le\ 
\sqrt{2}\,L \ \mathbb{E}_{\Xi_{MK}} \sup_{f \in \mathcal{F}} \sum_{i=1}^M \sum_{k=1}^K \xi_{ik}\, f_k(\vx_i),
\]
where $\{\xi_{ik}\}$ is an independent doubly indexed Rademacher sequence,  
and $f_k(\vx_i)$ denotes the $k$-th coordinate of $f(\vx_i)$.
\end{lemma}

\begin{proposition}\label{prop:boundcom}
Suppose Assumption~\ref{assump:act} holds. Define the class of norm-bounded $G$-invariant shallow networks as
\begin{align}
\label{eq:F_q}
\fF_Q := \left\{
  f^G(\vx; \Theta) = \frac{1}{m|G|} \sum_{i=1}^m a_i \sum_{g\in G} \sigma\!\bigl(\vw_i \cdot g\vx + b_i\bigr)
   \middle|~\Theta = \left(a_i, \vw_i, b_i\right)_{i=1}^m, ~m\ge 1, ~\|\Theta\|_{\mathcal{P}}\le Q \right\}.
\end{align}
where
\begin{align}
\label{eq:path_norm_2}
    \|\Theta\|_{\mathcal{P}}^2 \coloneqq \frac{1}{m} \sum_{i=1}^m \big(|a_i|\,(\|\vw_i\|_1 + |b_i| + 1)\big)^2
\end{align}
Then
\[
\hat{\tR}_S(\fF_Q) \ \le\ 4\,\gamma(\sigma)\,Q\,\sqrt{\frac{\log(2d+2)}{M}},
\]
where $S = \{\vx_1,\vx_2,\ldots,\vx_M\}$ is a set of samples in the domain $\Omega$, and $\gamma(\sigma)$ is defined in Assumption~\ref{assump:act}.
\end{proposition}

\begin{proof}
    Based on Assumption \ref{assump:act}, we know that for any $\epsilon\ge 0$, there is a ReLU neural network
    \begin{align*}
         g_K(x; \bar{\veta}):=\sum_{k=1}^K\bar{a}_k\operatorname{ReLU}(\bar{w}_kx+\bar{b}_k), \quad \text{where}~~\bar{\veta} \coloneqq \left(\bar{a}_k, \bar{w}_k, \bar{b}_k\right)_{k=1}^{K}
    \end{align*}
    such that    
    \begin{align*}
        \sup_{x\in\mathbb{R}}|\sigma(x)-g_K(x;\bar{\veta})| \leq \epsilon, \quad
        \sum_{k=1}^K |\bar{a}_k|(|\bar{w}_k|+|\bar{b}_k|)\leq \gamma(\sigma).
    \end{align*}
    Then for any $f^G(\vx; \Theta)=\frac{1}{m|G|}\sum_{i=1}^{m}
    a_i\sum_{g\in G}\sigma\!\bigl(\vw_i\cdot g\vx+b_i\bigr)\in\fF_Q$, where $\Theta\coloneqq (a_i, \vw_i, b_i)_{i=1}^m$,
    we have that
    \begin{align}
    \nonumber
        f^G(\vx; \Theta)=&\frac{1}{m|G|}\left[\sum_{i=1}^{m}
         a_i\sum_{g\in G}\sigma\!\bigl(\vw_i\cdot g\vx+b_i\bigr)-\sum_{i=1}^{m}
         a_i\sum_{g\in G}g_K\!\bigl(\vw_i\cdot g\vx+b_i\bigr)\right]\\
         \label{eq:rad_FQ_f}
         &+\frac{1}{m|G|}\sum_{i=1}^{m}
         a_i\sum_{g\in G}\sum_{k=1}^K\bar{a}_k\operatorname{ReLU}\!\bigl(\bar{w}_k(\vw_i\cdot g\vx+b_i)+\bar{b}_k\bigr)
    \end{align}
    For the first term, we know that 
    \[\sup_{\vx\in\sR^d}\left|\frac{1}{m|G|}\left[\sum_{i=1}^{m}
         a_i\sum_{g\in G}\sigma\!\bigl(\vw_i\cdot g\vx+b_i\bigr)-\sum_{i=1}^{m}
         a_i\sum_{g\in G}g_K\!\bigl(\vw_i\cdot g\vx+b_i\bigr)\right]\right|\le \frac{1}{m}\sum_{i=1}^m|a_i|\epsilon\le Q\epsilon.\]
         The second term of Eq.~\eqref{eq:rad_FQ_f} can be be expressed as a two-layer ReLU network:
         \begin{align*}
             H(\vx;\Theta,\bar{\veta})=\frac{1}{m}\sum_{i=1}^{m}\sum_{k=1}^Ka_i\bar{a}_k\fG\operatorname{ReLU}(\bar{w}_k\vw_i\cdot\vx+\bar{w}_kb_i+\bar{b}_k)
         \end{align*}
         A direct calculation shows that
         \begin{align*}
             \sum_{i=1}^m\sum_{k=1}^K|a_i||\bar{a}_k|(|\bar{w}_k|\|\vw_{i}\|_1+|\bar{w}_kb_i|+|\bar{b}_k|)\le &\sum_{i=1}^m\sum_{k=1}^K|a_i||\bar{a}_k|(\|\vw_{i}\|_1+|b_i|+1)(|\bar{w}_k|+|\bar{b}_k|)\\\le &\gamma(\sigma)Qm
         \end{align*}based on the Cauchy–Schwarz inequality. 
         
         Next, we define a new class of two-layer ReLU networks with norm constraints:
         \[\fM_R:=\left\{H(\vx; \hat{\Theta})=\sum_{p=1}^{K_*}\hat{a}_p\fG\operatorname{ReLU}(\hat{\vw}_p\cdot\vx+\hat{b}_p)\middle| ~ \hat{\Theta} = (\hat{a}_p, \hat{\vw}_p, \hat{b}_p)_{p=1}^{K_*}, ~K_*\ge 1, ~\|\hat{\Theta}\|_{\mathcal{P}_1} \le R\right\},\]
         where
         \begin{align*}
            \|\hat{\Theta}\|_{\mathcal{P}_1} \coloneqq \sum_{p=1}^{K_*} |\hat{a}_p|(\|\hat{\vw}_p\|_1+|\hat{b}_p|)
         \end{align*}
         We next bound the Rademacher complexity of $\fM_R$. For notational convenience, we write $a_p$ as $\hat{a}_p$, $\vw_p$ as $(\hat{\vw}_p,\hat{b}_p)$, and $\vx$ as $(\vx,1)$.

         \begin{align*}
             \hat{\tR}_S(\fM_R):=&\frac{1}{M}\mathbb{E}_{\Xi_M}\Biggl[\sup_{\|\Theta\|_{\fP_1}\le R}\sum_{i=1}^M\xi_i\,\sum_{p=1}^{K_*}a_p\|\vw_p\|_1\fG\operatorname{ReLU}\left(\frac{\vw_p}{\|\vw_p\|_1}\cdot\vx_i\right)\Biggr]\\\le &\frac{1}{M}\mathbb{E}_{\Xi_M}\Biggl[\sup_{\|\Theta\|_{\fP_1}\le R,\|\vu_p\|_1\le 1}\sum_{i=1}^M\xi_i\,\sum_{p=1}^{K_*}a_p\|\vw_p\|_1\fG\operatorname{ReLU}\left(\vu_p\cdot\vx_i\right)\Biggr]\\\le &\frac{1}{M}\mathbb{E}_{\Xi_M}\Biggl[\sup_{\|\Theta\|_{\fP_1}\le R}\,\sum_{p=1}^{K_*}|a_p|\|\vw_p\|_1\sup_{\|\vu\|_1\le 1}\left|\sum_{i=1}^M\xi_i\fG\operatorname{ReLU}\left(\vu\cdot\vx_i\right)\right|\Biggr]\\\le &\frac{R}{M}\mathbb{E}_{\Xi_M}\Biggl[\sup_{\|\vu\|_1\le 1}\left|\sum_{i=1}^M\xi_i\fG\operatorname{ReLU}\left(\vu\cdot\vx_i\right)\right|\Biggr]\\\le&\frac{2R}{M}\mathbb{E}_{\Xi_M}\Biggl[\sup_{\|\vu\|_1\le 1}\sum_{i=1}^M\xi_i\fG\operatorname{ReLU}\left(\vu\cdot\vx_i\right)\Biggr],
         \end{align*}
         where the last inequality uses the fact that $\hat{\tR}_S(\pm \mathcal{H}) \le 2\hat{\tR}_S (\mathcal{H})$ if $0\in\mathcal{H}$, where $\mathcal{H}$ in the above case is
         \begin{align}
             \mathcal{H} = \left\{\vx\mapsto \mathcal{G}\operatorname{ReLU}(\vu\cdot \vx): \|\vu\|_1\le 1\right\}\notag
         \end{align}
         
         Note that \[\fG\operatorname{ReLU}\left(\vu\cdot\vx_i\right)=\frac{1}{|G|}\sum_{g\in G}\operatorname{ReLU}(\vu\cdot g\vx)\]
         Since $G$ is a finite group, we can consider $\fG\operatorname{ReLU}$ as a function $h$ from $\sR^{|G|}\to \sR$, where \[h(y_1,y_2\ldots,y_{|G|})=\frac{1}{|G|}\sum_{s=1}^{|G|}\operatorname{ReLU}(y_s),\]with \[|h(\vy_1)-h(\vy_2)|\le \frac{1}{|G|}\|\vy_1-\vy_2\|_1\le \frac{1}{\sqrt{|G|}}\|\vy_1-\vy_2\|_2.\]         
         Therefore, based on the vector-contraction inequality for Rademacher complexities (Lemma \ref{lem:vector_contraction}), we have that \begin{align*}\mathbb{E}_{\Xi_M}\Biggl[\sup_{\|\vu\|_1\le 1}\sum_{i=1}^M\xi_i\fG\operatorname{ReLU}\left(\vu\cdot\vx_i\right)\Biggr]\le \frac{\sqrt{2}}{\sqrt{|G|}}\mathbb{E}_{\Xi_{M|G|}}\Biggl[\sup_{\|\vu\|_1\le 1}\sum_{i=1}^M\sum_{s=1}^{|G|}\xi_{i,s}\left(\vu\cdot g_s\vx_i\right)\Biggr]
         \end{align*}
         where $\xi_{i,s}$ are independent Rademacher variables. By Lemma~\ref{lem:B1}, we further have
         \[\frac{\sqrt{2}}{\sqrt{|G|}}\mathbb{E}_{\Xi_{M|G|}}\Biggl[\sup_{\|\vu\|_1\le 1}\sum_{i=1}^M\sum_{s=1}^{|G|}\xi_{i,s}\left(\vu\cdot g_s\vx_i\right)\Biggr]\le \frac{\sqrt{2}M|G|}{\sqrt{|G|}}\max_{g\in G,\vx\in\Omega}\|g\vx\|_\infty\sqrt{\frac{2\log(2d+2)}{M|G|}}.\] Summarizing the above bounds, we obtain
         \[\hat{\tR}_S(\fM_R)\le 2\sqrt{2}R\max_{g\in G,\vx\in\Omega}\|g\vx\|_\infty\sqrt{\frac{2\log(2d+2)}{M}}.\]
         We now return to bounding the Rademacher complexity of $\fF_Q$. Observe that
         \begin{align*}
             \hat{\tR}_S(\fF_Q)=&\frac{1}{M}\mathbb{E}_{\Xi_M}\Biggl[\sup_{f^G\in \fF_Q}\sum_{i=1}^M\xi_i\left(f^G(\vx_i;\Theta)-H(\vx_i;\Theta, \bar{\veta})+H(\vx_i;\Theta, \bar{\veta})\right)\Biggr]\\
             \le&\frac{1}{M}\mathbb{E}_{\Xi_M}\Biggl[\sup_{f^G\in \fF_Q}\sum_{i=1}^M\xi_i\left(f^G(\vx_i;\Theta)-H(\vx_i;\Theta, \bar{\veta})\right)\Biggr]+\frac{1}{M}\mathbb{E}_{\Xi_M}\Biggl[\sup_{H\in \fM_{\gamma(\sigma)Q}}\sum_{i=1}^M\xi_iH(\vx_i)\Biggr]\\\le &Q\epsilon+2\sqrt{2}\gamma(\sigma)Q\max_{g\in G,\vx\in\Omega}\|g\vx\|_\infty\sqrt{\frac{2\log(2d+2)}{M}}.
         \end{align*}As $\epsilon\to 0$, we finish the proof.
\end{proof}

This result shows that incorporating the $G$-invariant structure does not increase the Rademacher complexity of the function class. Consequently, the estimation (or statistical) error---arising from minimizing empirical risk rather than population risk (see Section~\ref{sec:generalization})---remains of the same order. To make this precise, we recall the following lemma.

\begin{lemma}[\cite{shalev2014understanding}]\label{lem:gapp}
Let $\fF\subset \mathbb{R}^\Omega$ be a hypothesis space.  
Assume that for any $f \in \fF$ and $\vx \in \Omega$, we have $|f(\vx)| \le L$.  
Then, for any $\delta > 0$, with probability at least $1 - \delta$ over the choice of 
$S = \{\vx_1, \vx_2, \ldots, \vx_M\} \subset \Omega$ drawn i.i.d.\ from distribution $\mu$,
\begin{align}
\label{eq:uniform_convergence}
\sup_{f\in\fF}\left| \frac{1}{M} \sum_{i=1}^M f(\vx_i) - \mathbb{E}_{\vx \sim \mu} \, f(\vx) \right|
\ \le\ 
2\, \mathbb{E}_{S \sim \mu^M} \hat{\tR}_S(\fF)
\;+\;
3L \sqrt{ \frac{2\log(4/\delta)}{M} },
\end{align}
\end{lemma}

\begin{remark}
\label{rmk:special_mu}
    Proposition~\ref{prop:boundcom} shows that the \emph{empirical} Rademacher complexity is unaffected by the $G$-invariant structure. Under further assumptions on the data distribution $\vx \sim \mu \in \mathcal{P}(\Omega)$ (see, e.g.,~\cite{pmlr-v54-sokolic17a,chen2025statistical,chen2023sample}), one might expect a $|G|$-dependent reduction in the \emph{expected} Rademacher complexity,
    \begin{align*}
        \tR_M(\mathcal{F}_Q) \coloneqq \mathbb{E}_{S\sim \mu^M}[\hat{\tR}_S(\mathcal{F}_Q)].
    \end{align*}
    where $S = (\vx_1,\ldots, \vx_M)$. Since no assumptions of this type are imposed in the present paper, the non-increase guarantee is the sharpest conclusion available.
\end{remark}

\subsection{Generalization error}
\label{sec:generalization}

Let $f_* \in \mathcal{B}_\Gamma^G(\Omega)$ denote the (unknown) ground-truth $G$-invariant Barron function. Our goal is to learn $f_*$ from noisy observations at data points $\vx_1, \ldots, \vx_M$, where $\vx_i \stackrel{\text{i.i.d.}}{\sim} \mu$ for some unknown distribution $\mu$ on $\Omega$. We formalize this setting as follows.  

\begin{assumption}\label{assump:bound noise}
The training data $S = (\vx_i, y_i)_{i=1}^M$ are drawn i.i.d.\ from the joint distribution $\mathcal{D}$
\begin{align}
\label{eq:joint_distribution}
\vx \sim \mu \in \mathcal{P}(\Omega), 
\qquad y = f_*(\vx) + \varepsilon_0,
\end{align}
where the observation noise $\varepsilon_0$ is bounded and satisfies
\[
\mathbb{P}(|\varepsilon_0| < \tau) = 1, 
\qquad \mathbb{E}[\varepsilon_0^2] = \tau_0, 
\qquad \mathbb{E}[\varepsilon_0 \mid \vx] = 0 \quad \forall \vx \in \Omega,
\]
for constants $\tau, \tau_0 > 0$. Without loss of generality, we set $\tau = 1$.
\end{assumption}

To learn $f_*$ using a statistical learning approach, we consider {\color{black}training a width-$m$ $G$-invariant shallow network $f_m^G(\cdot; \Theta^m)$, where the superscript $m$ indicates that $\Theta^m$ are the parameters of the width-$m$ network}. The associated continuous loss function (population risk) is defined by
\begin{align}
\label{eq:population_loss}
    \color{black}L(\Theta^m) \;=\; \mathbb{E}_{(\vx, y) \sim \mathcal{D}}\Big[\, \ell\big(f^G_m(\vx;\Theta^m),\, y\big)\,\Big],
\end{align}
where $\mathcal{D}$ is given by Eq.~\eqref{eq:joint_distribution}, $\ell$ is taken to be the squared loss {\color{black}$\ell (y, y')= (y-y')^2$}, and $f_m^G(\vx; \Theta^m)$ is again given by
\begin{align}
    \color{black}f^G_m(\vx; \Theta^m) = \frac{1}{m|G|} \sum_{i=1}^m a_i \sum_{g\in G} \sigma\!\bigl(\vw_i \cdot g\vx + b_i\bigr), \quad \Theta^m = (a_i, \vw_i, b_i)_{i=1}^m
\end{align}
{\color{black}Simple calculation leads to
\begin{align}
    \nonumber
    L(\Theta^m) & = \EE_{\vx, y, \varepsilon_0}\left|f_m^G(\vx; \Theta^m)-y\right|^2\\
    \nonumber
    & = \EE_{\vx\sim \mu}~\EE_{y, \varepsilon_0|\vx}\left|f_m^G(\vx; \Theta^m)-f_*(\vx) -\varepsilon_0\right|^2\\
    \nonumber
    & = \EE_{\vx\sim\mu}~\EE_{\varepsilon_0|\vx}[\varepsilon_0^2]+\EE_{\vx\sim\mu}\left|f_m^G(\vx; \Theta^m)-f_*(\vx)\right|^2\\
    \label{eq:L_and_tau}
    & = \tau_0 + \int_{\Omega}\left|f_m^G(\vx; \Theta^m)-f_*(\vx)\right|^2 \D\mu(x)
\end{align}
}

In practice, the population loss~\eqref{eq:population_loss} is not avaiable (due to the unknown joint distribution $\mathcal{D}$), we can only work with the empirical loss based on $M$ training data $(\vx_i, y_i)_{i=1}^M$ sampled from $\mathcal{D}$
\[
\color{black}\hat{L}_M(\Theta^{\color{black}m}) = \frac{1}{M} \sum_{i=1}^M \ell \big( f^G_m(\vx_i; \Theta^{\color{black}m}), y_i \big),
\]
To further remove the constraint on the parameter norm size, in this paper, we will work with the regularized empirical risk, defined as
\[
J_\lambda(\Theta^{\color{black}m})
\;:=\;
\hat{L}_{M}(\Theta^{\color{black}m})
\;+\;
\lambda (\|\Theta^{\color{black}m}\|_{\fP}^2+1) ,
\]
where the norm $\|\Theta^{\color{black}m}\|_\mathcal{P}$ is given by Eq.~\eqref{eq:path_norm_2}. The +1 term at the right-hand side is included only to simplify the proof, which does not affect minimizer. The corresponding estimator is
\[
\color{black}\hat{\Theta}_{M,\lambda}^m
\;:=\;
\arg\min_{\Theta^m} J_\lambda(\Theta^{m}).
\]
Our goal is to bound
\[
\color{black}\int_{\Omega} \big| f^G_m(\vx; \hat{\Theta}_{M,\lambda}^m) - f_*(\vx) \big|^2 \, \D\mu(\vx).
\]


\subsubsection{Estimation of generalization error}

Consider the norm-constrained $G$-invariant {\color{black}width-$m$} shallow network class
\begin{align*}
\fF_{\color{black}Q,m} := \left\{
  f^G_{\color{black}m}(\vx; \Theta^{\color{black}m}) = \frac{1}{m|G|} \sum_{i=1}^m a_i \sum_{g\in G} \sigma\!\bigl(\vw_i \cdot g\vx + b_i\bigr)
   \middle|~\Theta^{\color{black}m} = \left(a_i, \vw_i, b_i\right)_{i=1}^m, ~\|\Theta^{\color{black}m}\|_{\mathcal{P}}\le Q \right\},
\end{align*}
where
\begin{align*}
    \|\Theta^{\color{black}m}\|_{\mathcal{P}}^2 \coloneqq \frac{1}{m} \sum_{i=1}^m \big(|a_i|\,(\|\vw_i\|_1 + |b_i| + 1)\big)^2
\end{align*}
{\color{black}Note that $\mathcal{F}_{Q, m}\subset \mathcal{F}_Q$ defined in Eq.~\eqref{eq:F_q} of Proposition~\ref{prop:boundcom}, as the width of networks in $\mathcal{F}_Q$ is not fixed.} The following lemma provides uniform bounds, on $\Omega$, for elements in ${\color{black}\fF_{Q, m}}$ and in the $\sigma$-activated $\Gamma$-Barron space $\mathcal{B}_{\Gamma}(\Omega)$.

\begin{lemma}\label{lem:uniform bound}
Under Assumption~\ref{assump:bounded_omega}, suppose $\sigma$ is Lipschitz continuous with constant $L_\sigma$. Then:
\begin{enumerate}[(i)]
    \item For any $f^G_{\color{black}m} \in \fF_{Q,\color{black}m}$, 
    \[
    \left|f^G_{\color{black}m}(\vx)\right| \ \le\ (L_\sigma + |\sigma(0)|)Q, \quad \forall \vx \in \Omega.
    \]
    \item For any $f \in \fB_\Gamma(\Omega)$ with $\|f\|_{\fB_\Gamma} \le Q$, 
    \[
    |f(\vx)| \ \le\ (L_\sigma + |\sigma(0)|)Q, \quad \forall \vx \in \Omega.
    \]
\end{enumerate}
\end{lemma}

\begin{proof}
(i) For any $\vx \in \Omega$, 
\begin{align*}
\left|f^G_{\color{black}m}(\vx)\right| & = \left| \frac{1}{m|G|} \sum_{i=1}^m a_i \sum_{g\in G} \sigma\!\bigl(\vw_i \cdot g\vx + b_i\bigr) \right| \\
&\le \frac{1}{m|G|} \sum_{i=1}^m |a_i| \sum_{g\in G} \big| \sigma\!\bigl(\vw_i \cdot g\vx + b_i\bigr) - \sigma(0) \big| \ +\ \frac{|\sigma(0)|}{m} \sum_{i=1}^m |a_i| \\
& \le \frac{1}{m} \sum_{i=1}^m |a_i|\, L_\sigma (\|\vw_i\|_1 + |b_i|) \ +\ \frac{|\sigma(0)|}{m} \sum_{i=1}^m |a_i| \\
& \le (L_\sigma + |\sigma(0)|) \sqrt{\frac{1}{m} \sum_{i=1}^m \big(|a_i|\,(\|\vw_i\|_1 + |b_i| + 1)\big)^2} \\
& \le (L_\sigma + |\sigma(0)|) Q.
\end{align*}

(ii) If $f \in \fB_\Gamma(\Omega)$ with $\|f\|_{\fB_\Gamma} \le Q$, 
then there exists a probability measure $\rho\in\Gamma$ on $\mathbb{R}^{d+2}$ such that
\[
f(\vx) = \mathbb{E}_{(a,\vw,b)\sim\rho} \big[ a\,\sigma(\vw \cdot \vx + b) \big], 
\quad \forall \vx \in \Omega,
\]
and
\[
\sqrt{ \mathbb{E}_{(a,\vw,b)\sim\rho} \,|a|^2 \, (\|\vw\|_1 + |b| + 1)^2 } 
\ \le\ (1+\varepsilon) Q.
\]
By Lipschitz continuity,
\begin{align*}
|f(\vx)| 
&\le \mathbb{E}_{\rho} \big[ |a|\, |\sigma(\vw \cdot \vx + b) - \sigma(0)| \big] 
  + |\sigma(0)|\,\mathbb{E}_{\rho}[|a|] \\
&\le (L_\sigma + |\sigma(0)|) \,\mathbb{E}_{\rho} \big[ |a|\,(\|\vw\|_1 + |b| + 1) \big] \\
&\le (L_\sigma + |\sigma(0)|) \sqrt{ \mathbb{E}_{\rho} \,|a|^2 (\|\vw\|_1 + |b| + 1)^2 } \\
&\le (L_\sigma + |\sigma(0)|)\,(1+\varepsilon)Q.
\end{align*}
Letting $\varepsilon \to 0$ completes the proof.
\end{proof}


Now we can prove the following proposition.

\begin{proposition}\label{prop:gap-theta}
Suppose Assumptions~\ref{assump:bounded_omega}-\ref{assump:bound noise} all hold. 
For any $\delta > 0$, with probability at least $1 - \delta$ over the choice of the training set $S = (\vx_i, y_i)_{i=1}^M\sim\mathcal{D}^M$, we have
\[
\sup_{\|\Theta^{\color{black}m}\|_{\fP} \le Q} \big| L(\Theta^{\color{black}m}) - \hat{L}_{M}(\Theta^{\color{black}m}) \big|
\ \le\ 
C_\sigma\,(Q+1)^2\left[\sqrt{\frac{\log(2d+2)}{M}}
\;+\;
\sqrt{ \frac{2\log(2/\delta)}{M} }\right],
\]
where $C_\sigma>0$ depends only on $\|f_*\|_{\fB_\Gamma}$ and $\sigma$ (through $L_\sigma$, $|\sigma(0)|$, $\gamma(\sigma)$), but not on $M,d,Q$.
\end{proposition}

\begin{proof}
Since the Rademacher complexity is translation-invariant, the class  
\[
\mathcal{H}=\big\{ f^G_{\color{black}m}(\vx; \Theta) - y \ \big|\ \|\Theta^{\color{black}m}\|_{\fP} \le Q \big\}
\]
has the same empirical Rademacher complexity {\color{black} as $\mathcal{F}_{Q, m}$, which is bounded by that of $\mathcal{F}_Q$} in Proposition~\ref{prop:boundcom}.  

By Assumption~\ref{assump:bound noise} and Lemma~\ref{lem:uniform bound}, for all $\vx\in\Omega$ and $\|\Theta^{\color{black}m}\|_{\fP} \le Q$,
\[
\big| f^G_{\color{black}m}(\vx; \Theta) - y \big| \ \le\ B,
\quad\text{where}\quad
B := \big(L_\sigma+|\sigma(0)|\big)\big(Q+\|f_*\|_{\fB_\Gamma}\big) + 1 .
\]

Note that the squared loss $\phi(z)=z^2$ is $2B$-Lipschitz on $[-B,B]$, since $\phi'(z)=2z$. Applying Lemma~\ref{lem:vector_contraction} with $K=1$, together with Proposition~\ref{prop:boundcom}, 
yields the following bound for the Rademacher term in Eq.~\eqref{eq:uniform_convergence}:
\[
  8\,B\,\gamma(\sigma)\,Q\,\sqrt{\frac{\log(2d+2)}{M}}.
\]
Here, the factor $\sqrt{2}$ in Lemma~\ref{lem:vector_contraction} can be dropped in the scalar case $K=1$ (Talagrand–Ledoux contraction), as noted in~\cite{maurer2016vector}.
For the concentration term in Eq.~\eqref{eq:uniform_convergence}, we use  $L=B^2$, obtaining
\[
3B^2\sqrt{\frac{2\log(2/\delta)}{M}}.
\]
Finally, note that $B \le C_1\,(Q+1)$ with 
\[
C_1 := \big(L_\sigma+|\sigma(0)|\big)\big(\|f_*\|_{\fB_\Gamma}+1\big) + 1.
\]
By absorbing constants into
\[
C_\sigma := \max\left\{ 8\,\gamma(\sigma)\,C_1,\ 3C_1^2 \right\},
\]
we obtain the desired bound.
\end{proof}

\begin{proposition}\label{prop:gap theta1}
Suppose Assumptions~\ref{assump:bounded_omega}-\ref{assump:bound noise} all hold. 
For any $\delta > 0$, with probability at least $1 - \delta$ over the choice of the training set $S = (\vx_i, y_i)_{i=1}^M\sim\mathcal{D}^M$, the following uniform bound holds for \emph{all} parameters $\Theta^{\color{black}m}$:
\[
\big| L(\Theta^{\color{black}m}) - \hat{L}_{M}(\Theta^{\color{black}m}) \big|
\ \le\ 
D_\sigma\,\big(\|\Theta^{\color{black}m}\|_{\fP}^2+1\big)\,\left[\sqrt{\frac{\log(2d+2)}{M}}
\;+\;
\sqrt{ \frac{2\log\!\big(2\,C_\zeta\,(\|\Theta^{\color{black}m}\|_{\fP}+1)^2/\delta\big)}{M} }\right],
\]
where $C_\zeta := \sum_{q=1}^\infty \frac{1}{q^2}=\pi^2/6$, and $D_\sigma=4C_\sigma$ depends only on $\|f_*\|_{\fB_\Gamma}$ and $\sigma$.
\end{proposition}

\begin{proof}
Fix $Q\in\mathbb{N}$ and set $\delta_Q := \delta/(C_\zeta Q^2)$. 
By Proposition~\ref{prop:gap-theta}, with probability at least $1-\delta_Q$,
\[
\sup_{\|\Theta^{\color{black}m}\|_{\fP} \le Q} \big| L(\Theta^{\color{black}m}) - \hat{L}_{M}(\Theta^{\color{black}m}) \big|
\ \le\ 
C_\sigma\,(Q+1)^2\,\left[\sqrt{\frac{\log(2d+2)}{M}}
\;+\;
\sqrt{ \frac{2\log(2/\delta_Q)}{M} }\right].
\]
By the \emph{union bound} (Boole’s inequality) and $\sum_{Q=1}^\infty \delta_Q=\delta$, the above holds simultaneously for all $Q\in\mathbb{N}$ with probability at least $1-\delta$.

Given arbitrary $\Theta^{\color{black}m}$, take $Q_0:=\lceil \|\Theta^{\color{black}m}\|_{\fP}\rceil$ so that $\|\Theta^{\color{black}m}\|_{\fP}\le Q_0$ and $Q_0\le \|\Theta^{\color{black}m}\|_{\fP}+1$. Plugging $Q=Q_0$ yields
\[
\big| L(\Theta^{\color{black}m}) - \hat{L}_{M}(\Theta^{\color{black}m}) \big|
\ \le\ 
C_\sigma\,(\|\Theta^{\color{black}m}\|_{\fP}+2)^2\,\left[\sqrt{\frac{\log(2d+2)}{M}}
\;+\;
\sqrt{ \frac{2\log\!\big(2\,C_\zeta\,(\|\Theta^{\color{black}m}\|_{\fP}+1)^2/\delta\big)}{M} }\right].
\]
Finally, $(\|\Theta^{\color{black}m}\|_{\fP}+2)^2 \le 4(\|\Theta^{\color{black}m}\|_{\fP}^2+1)$ gives the stated form with $D_\sigma:=4C_\sigma$.
\end{proof}

Now we can state two important bounds based on Proposition~\ref{prop:gap theta1}.
\begin{proposition}\label{prop: two bounds}
Let $f_*\in \fB_\Gamma^G(\Omega)$ be the target $G$-invariant Barron function, and let ${\color{black}f^G_m(\cdot; \Theta_*^{m})}$ be the {\color{black}width-$m$} shallow network approximating $f_*$, as given in Eq.~\eqref{eq:G_approx_m} of Theorem~\ref{app}. If
\[
\lambda \ \ge \ D_\sigma\,\sqrt{\frac{\log(2d+2)}{M}},
\]
then with probability at least $1-\delta$ we have:
\begin{align}
\label{eq:bound_on_J_theta_*}
J_\lambda(\Theta_*^m)
\ \le\ 
L(\Theta_*^m) + 2\lambda \,(2\|f_*\|^2_{\fB_\Gamma}+1)
+D_\sigma\,(2\|f_*\|^2_{\fB_\Gamma}+1)
\sqrt{ \frac{2\log\!\big(2\,C_\zeta\,(\sqrt{2}\|f_*\|_{\fB_\Gamma}+1)^2/\delta\big)}{M} },
\end{align}
and
\begin{align}
\label{eq:bound_on_J_theta_hat}
\left\{
\begin{aligned}
& J_\lambda(\hat{\Theta}^m_{M,\lambda}) \ \le\ J_\lambda(\Theta_*^m),\\
& \|\hat{\Theta}^m_{M,\lambda}\|^2_{\fP}
\ \le\ \frac{L(\Theta^m_*)}{\lambda} 
+(2\|f_*\|^2_{\fB_\Gamma}+1)
\left(\sqrt{2\log\!\big(2\,C_\zeta\,(\sqrt{2}\|f_*\|_{\fB_\Gamma}+1)^2/\delta\big) }+2\right).
\end{aligned}\right.
\end{align}
\end{proposition}

\begin{proof}
(i)  
To establish Eq.~\eqref{eq:bound_on_J_theta_*}, observe that Proposition~\ref{prop:gap theta1}, together with the bound $\|\Theta_*^m\|_{\fP} \le \sqrt{2}\,\|f_*\|_{\fB_\Gamma}$, implies that
\begin{align*}
J_\lambda(\Theta_*^m)
&= \hat{L}_{M}(\Theta_*^m) + \lambda  (\|\Theta_*^m\|_{\fP}^2+1) \\
&\le L(\Theta_*^m) + \lambda \,(2\|f_*\|^2_{\fB_\Gamma}+1)
+ \\
&\hspace{5em} D_\sigma\,(2\|f_*\|^2_{\fB_\Gamma}+1)\left(\sqrt{\frac{\log(2d+2)}{M}}
\;+\;
\sqrt{ \frac{2\log\!\big(2\,C_\zeta\,(\sqrt{2}\|f_*\|_{\fB_\Gamma}+1)^2/\delta\big)}{M} }\right)\\
&\le L(\Theta_*^m) + 2\lambda \,(2\|f_*\|^2_{\fB_\Gamma}+1)
+D_\sigma\,(2\|f_*\|^2_{\fB_\Gamma}+1)
\sqrt{ \frac{2\log\!\big(2\,C_\zeta\,(\sqrt{2}\|f_*\|_{\fB_\Gamma}+1)^2/\delta\big)}{M} },
\end{align*}
where the last inequality is based on \(
\lambda \ \ge \ D_\sigma\,\sqrt{\frac{\log(2d+2)}{M}}
\).

(ii)  
To prove Eq.~\eqref{eq:bound_on_J_theta_hat}, note that the first inequality \(J_\lambda(\hat{\Theta}^m_{M,\lambda}) \ \le\ J_\lambda(\Theta^m_*)\) follows directly from the definition of $\hat{\Theta}^m_{M,\lambda}$.  
For the second, we use
\[\lambda\|\hat{\Theta}^m_{M,\lambda}\|^2_{\fP}\le J_\lambda(\hat{\Theta}^m_{M,\lambda}) \ \le\ J_\lambda(\Theta^m_*)\]
Then based on the results in (i), we have
\begin{align*}
\|\hat{\Theta}^m_{M,\lambda}\|^2_{\fP}
&\le \frac{J_\lambda(\Theta^m_*)}{\lambda } \\
&\le \frac{L(\Theta^m_*)}{\lambda} + 2 \,(2\|f_*\|^2_{\fB_\Gamma}+1)
+(2\|f_*\|^2_{\fB_\Gamma}+1)
\frac{\sqrt{2\log\!\Big(2\,C_\zeta\,(\sqrt{2}\|f_*\|_{\fB_\Gamma}+1)^2/\delta\Big)}}{\sqrt{\log(2d+2)}} \\
&\le \frac{L(\Theta^m_*)}{\lambda} + 2 \,(2\|f_*\|^2_{\fB_\Gamma}+1)
+(2\|f_*\|^2_{\fB_\Gamma}+1)
{\sqrt{2\log\!\Big(2\,C_\zeta\,(\sqrt{2}\|f_*\|_{\fB_\Gamma}+1)^2/\delta\Big)}} ,
\end{align*}
where in the last two steps we used $\sqrt{\log(2d+2)} \ge 1$ and the lower bound on $\lambda$.
\end{proof}

\begin{theorem}\label{thm:gen}
Suppose Assumptions~\ref{assump:bounded_omega}-\ref{assump:bound noise} all hold.
Then for $f_*\in\fB_{\Gamma}^G(\Omega)$ and
\begin{align}
\label{eq:cond_on_lambda}
\lambda \ \ge \ \max\left\{ D_\sigma\,\sqrt{\frac{\log(2d+2)}{M}}, \ \frac{3 (L_\sigma+|\sigma(0)|)^2 \delta_{G,\Gamma,\sigma} \|f_*\|_{\fB_\Gamma}^{2}}{m}, \right\}
\end{align}
we have for any \(\delta>0\), with probability at least \(1-\delta\) over the choice of training set $S = (\vx_i, y_i)_{i=1}^M\sim\mathcal{D}^M$,
\begin{align}
\nonumber
    \int_\Omega |f_m^G(\vx;\hat{\Theta}^m_{M,\lambda}) - f_*(\vx)|^2\D\mu\le
    & \frac{3 (L_\sigma+|\sigma(0)|)^2 \delta_{G,\Gamma,\sigma} \|f_*\|_{\fB_\Gamma}^{2}}{m}\\
    \nonumber
    &+2\lambda \,(2\|f_*\|^2_{\fB_\Gamma}+1)
    + D_\sigma\,(2\|f_*\|^2_{\fB_\Gamma}+1)(\Lambda_f(\delta)-2)
    {M^{-\frac{1}{2}}}\\
    \label{eq:main_theorem}
    &+ D_\sigma\,(R_f(\delta)+\frac{\tau_0}{\lambda}+1)
    \sqrt{ \frac{2\log\!\left( \frac{4\,C_\zeta\,(R_f(\delta)+\frac{\tau_0}{\lambda}+1)}{\delta} \right)}{M} }.
\end{align}
where
\[
\Lambda_f(\delta) := \sqrt{ 2\log\!\left( \frac{2\,C_\zeta\,(\sqrt{2}\|f_*\|_{\fB_\Gamma}+1)^2}{\delta} \right) } + 2,
\quad
R_f(\delta) := 1+(2\|f_*\|_{\fB_\Gamma}^2+1)\,\Lambda_f(\delta).
\]
\end{theorem}

\begin{proof}
From Proposition~\ref{prop: two bounds}, with probability at least \(1-\tfrac{\delta}{2}\),
\begin{align}
J_\lambda(\hat \Theta^m_{M,\lambda})
&\le L(\Theta^m_*) + 2\lambda \,(2\|f_*\|^2_{\fB_\Gamma}+1)
+ D_\sigma\,(2\|f_*\|^2_{\fB_\Gamma}+1)\,
\sqrt{ \frac{2\log\!\left( \tfrac{2\,C_\zeta\,(\sqrt{2}\|f_*\|_{\fB_\Gamma}+1)^2}{\delta} \right)}{M} },
\label{highevent1}
\end{align}where $\Theta^m_*$ are the parameters of the neural network defined in Theorem~\ref{app}.

Similarly,
\begin{equation}\label{eq:param-norm-step}
\|\hat{\Theta}^m_{M,\lambda}\|^2_{\fP}
\ \le\ \frac{L(\Theta^m_*)}{\lambda} 
+ (2\|f_*\|^2_{\fB_\Gamma}+1)\,\Lambda_f(\delta)
\ \le\ R_f(\delta)+\frac{\tau_0}{\lambda}.
\end{equation}
The last inequality in \eqref{eq:param-norm-step} follows from Eq.~\eqref{eq:L_and_tau}, Theorem~\ref{app} and the lower bound on \(\lambda\). Specifically,
\begin{align*}
\frac{L(\Theta_*^m)}{\lambda} 
&= \frac{\int_{\Omega}|f(\vx;\Theta^m_*)-f_*(\vx)|^2\,\D \mu}{\lambda}
  +\frac{\tau_0}{\lambda} \\
&\le \frac{3 (L_\sigma+|\sigma(0)|)^2 \,\delta_{G,\Gamma,\sigma}\,\|f_*\|_{\fB_\Gamma}^{2}}{m\,\lambda}  
  + \frac{\tau_0}{\lambda}
\ \le\ 1+\frac{\tau_0}{\lambda}.
\end{align*}
Combining this with the first bound in \eqref{eq:param-norm-step} gives
\[
\|\hat{\Theta}^m_{M,\lambda}\|^2_{\fP}
\le R_f(\delta)+\frac{\tau_0}{\lambda},
\]
as required.

From Proposition \ref{prop:gap theta1}, with probability at least \(1-\frac{\delta}{2}\),
\begin{align}
L(\hat \Theta^m_{M,\lambda})
&\le \hat L_{M}(\hat \Theta^m_{M,\lambda})
+ D_\sigma\,\big(\|\hat \Theta^m_{M,\lambda}\|_{\fP}^2+1\big)
\left(\sqrt{ \frac{2\log\!\left( \frac{2\,C_\zeta\,(\|\hat \Theta^m_{M,\lambda}\|_{\fP}+1)^2}{\delta} \right)}{M} }+\,\sqrt{\frac{\log(2d+2)}{M}}\right)
\notag \\
&\le\hat L_{M}(\hat \Theta^m_{M,\lambda})+\lambda\,\big(\|\hat \Theta^m_{M,\lambda}\|_{\fP}^2+1\big)+D_\sigma\,\big(\|\hat \Theta^m_{M,\lambda}\|_{\fP}^2+1\big)
\sqrt{ \frac{2\log\!\left( \frac{2\,C_\zeta\,(\|\hat \Theta^m_{M,\lambda}\|_{\fP}+1)^2}{\delta} \right)}{M} }\notag \\
&=J_\lambda(\hat \Theta^m_{M,\lambda})+D_\sigma\,\big(\|\hat \Theta^m_{M,\lambda}\|_{\fP}^2+1\big)
\sqrt{ \frac{2\log\!\left( \frac{2\,C_\zeta\,(\|\hat \Theta^m_{M,\lambda}\|_{\fP}+1)^2}{\delta} \right)}{M} }\notag \\
&\le J_\lambda(\hat \Theta^m_{M,\lambda})
+ D_\sigma\,\left(R_f(\delta)+\frac{\tau_0}{\lambda}+1\right)
\sqrt{ \frac{2\log\!\left( \frac{4\,C_\zeta\,(R_f(\delta)+\frac{\tau_0}{\lambda}+1)}{\delta} \right)}{M} }.\label{highevent2}
\end{align}
Combining the two high-probability events \eqref{highevent1} and \eqref{highevent2}, we obtain
\begin{align}
    L(\hat \Theta^m_{M,\lambda})
    \;\le\;& L(\Theta^m_*) 
    + 2\lambda \,(2\|f_*\|^2_{\fB_\Gamma}+1) \notag\\
    &+ D_\sigma\,(2\|f_*\|^2_{\fB_\Gamma}+1)(\Lambda_f(\delta)-2)\,M^{-\tfrac{1}{2}} \notag\\
    &+ D_\sigma\Bigl(R_f(\delta)+\tfrac{\tau_0}{\lambda}+1\Bigr)\,
    \sqrt{ \tfrac{2\log\!\Bigl( \tfrac{4\,C_\zeta\,(R_f(\delta)+\tfrac{\tau_0}{\lambda}+1)}{\delta} \Bigr)}{M} },
    \label{eq: noise no}
\end{align} with probability at least \(1-\delta\).

On the other hand, by Theorem~\ref{app} and Eq.~\eqref{eq:L_and_tau}, we can decompose the population risk as
\begin{align}
    L(\hat{\Theta}^m_{M,\lambda})
    &= \int_\Omega \bigl|f^G_m(\vx;\hat{\Theta}^m_{M,\lambda}) - f_*(\vx)\bigr|^2\,\D\mu \;+\; \tau_0, \notag\\
    L(\Theta^m_*)
    &= \int_\Omega \bigl|f^G_m(\vx;\Theta^m_*) - f_*(\vx)\bigr|^2\,\D\mu \;+\; \tau_0 \le \frac{3 (L_\sigma+|\sigma(0)|)^2\, \delta_{G,\Gamma,\sigma}\,\|f_*\|_{\fB_\Gamma}^{2}}{m} \;+\; \tau_0.
    \label{eq:noise omit}
\end{align}

Finally, combining \eqref{eq: noise no} and \eqref{eq:noise omit} yields the desired bound for $\int_\Omega \bigl|f^G_m(\vx;\hat{\Theta}^m_{M,\lambda}) - f(\vx)\bigr|^2\,\D\mu$. This concludes the proof.
\end{proof}

We now discuss the benefit of using a $G$-invariant network to learn a $G$-invariant target function $f_*\in\fB_{\Omega}^G$, as described in Eq.~\eqref{eq:main_theorem} of Theorem~\ref{thm:gen}. The main advantage arises from the factor $\delta_{G,\Gamma,\sigma} \le 1$, which influences the generalization error in the following ways:
\begin{itemize}
    \item \textbf{Approximation error reduction.} The first term in Eq.~\eqref{eq:main_theorem},
    \begin{align*}
        \frac{3 (L_\sigma+|\sigma(0)|)^2 \delta_{G,\Gamma,\sigma} \|f_*\|_{\fB_\Gamma}^{2}}{m},
    \end{align*}
    is reduced by a factor of $\delta_{G,\Gamma,\sigma}$.
    \item \textbf{More flexible choice of regularization.} 
    To reduce the second term,
    \begin{align*}
        2\lambda\,(2\|f_*\|^2_{\fB_\Gamma}+1),
    \end{align*}
    one needs a smaller $\lambda$. Incorporating $G$-invariance permits a broader choice of the regularization parameter while still achieving optimal generalization error, since
        \[
        \lambda \ \ge \ \max\left\{ D_\sigma\,\sqrt{\frac{\log(2d+2)}{M}}, \ \frac{3 (L_\sigma+|\sigma(0)|)^2 \,\delta_{G,\Gamma,\sigma}\, \|f_*\|_{\fB_\Gamma}^{2}}{m} \right\},
        \]
    \item \textbf{Estimation error.} Without additional assumptions on the data distribution $\mu \in \fP(\Omega)$, the estimation error
        \begin{align*}
            D_\sigma\,(2\|f_*\|^2_{\fB_\Gamma}+1)(\Lambda_f(\delta)-2)
            {M^{-\frac{1}{2}}}
        \end{align*}
    is unaffected by the incorporation of $G$-invariance. However, as noted in Remark~\ref{rmk:special_mu}, stronger assumptions on $\mu$, such as those in~\cite{pmlr-v54-sokolic17a,chen2025statistical,chen2023sample}, can yield a $|G|$-dependent reduction of the statistical error.
\end{itemize}

In summary, the effect of incorporating group invariance on generalization error is governed by the factor $\delta_{G,\Gamma,\sigma}$. When $\delta_{G,\Gamma,\sigma} \ll 1$, a significant gain in generalization can be achieved. When it is close to one, little to no improvement is expected, though Theorem~\ref{app} ensures that no deterioration occurs, as $\delta_{G,\Gamma,\sigma} \le 1$.

\section{Discussion on $\delta_{G,\Gamma,\sigma}$}\label{sec: example}
The parameter $\delta_{G,\Gamma,\sigma}$ quantifies how effectively group symmetry reduces redundancy in the representation.
In this section, we provide several examples to illustrate when $\delta_{G,\Gamma,\sigma}\ll 1$ and when it is not.  
The key idea for achieving $\delta_{G,\Gamma,\sigma}\ll 1$ is to have the probability set $\Gamma$ and the activation function $\sigma$ in such a way that, under any probability measure $\rho \in \Gamma$, the functions $\sigma(\vw \cdot g_1 \vx + b)$ and $\sigma(\vw \cdot g_2 \vx + b)$ have only a small overlap for distinct $g_1\neq g_2\in G$.  
In this favorable case, we typically obtain $\delta_{G,\Gamma,\sigma}\simeq \tfrac{1}{|G|}$.  
On the other hand, if there is significant overlap among these functions, then $\delta_{G,\Gamma,\sigma}\simeq 1$, and no such improvement can be expected. Here the notation $\simeq$ means that the quantity is of the same order of magnitude.

\subsection{ReLU-shaped and S-shaped activation functions}

One fundamental class of activation functions is the family of \emph{ReLU-shaped} functions, including
\begin{itemize}
    \item $\mathrm{ReLU}(x) = \max\{x,0\}$
    \item $\mathrm{Softplus}(x) = \log(1 + e^x)$
    \item $\mathrm{SiLU}(x) = x \cdot \mathrm{Sigmoid}(x)$ 
    \item $\mathrm{Mish}(x) = x \cdot \tanh(\mathrm{Softplus}(x))$
    \item $\mathrm{GELU}(x) = x \cdot \int_{-\infty}^x \frac{1}{\sqrt{2\pi}} e^{-t^2/2} \, dt$
    \item SELU and CELU; see Eq.~\eqref{eq:celu_selu}.
\end{itemize}

Another important class is that of \emph{S-shaped} activation functions, such as
\[
\mathrm{Sigmoid}(x) = \frac{1}{1 + e^{-x}}, 
\quad 
\mathrm{Tanh}(x) = \frac{e^{x} - e^{-x}}{e^{x} + e^{-x}}, 
\quad \cdots
\]
A common feature of both classes is that their outputs are very small, and close to zero, whenever $x \leq 0$. The central idea for keeping $\delta_{G,\Gamma,\sigma}$ small is therefore to ensure that $\vw \cdot g_1 \vx$ and $\vw \cdot g_2 \vx$ have little overlap in their positive regions, which holds with high probability under $\rho \in \Gamma$.

For example, let $\Omega\subset\mathbb{R}^1$ and take $G = \{e, r\}$ to be the reflection group defined by $rx = -x$. Consider the affine function $wx+b$ with $w\in\mathbb{R}$ and $b\le 0$. n this case, the graphs of $\operatorname{ReLU}(w x + b)$ and $\operatorname{ReLU}(-w x + b)$ have disjoint supports. For instance, when $w=1$ and $b=-1$, the sketches of $\operatorname{ReLU}(x-1)$ and $\operatorname{ReLU}(-x-1)$ are shown in Fig.~\ref{fig:ReLU}.  

\begin{figure}[h!]
    \centering
    \begin{subfigure}[b]{0.45\textwidth}
        \includegraphics[width=\textwidth]{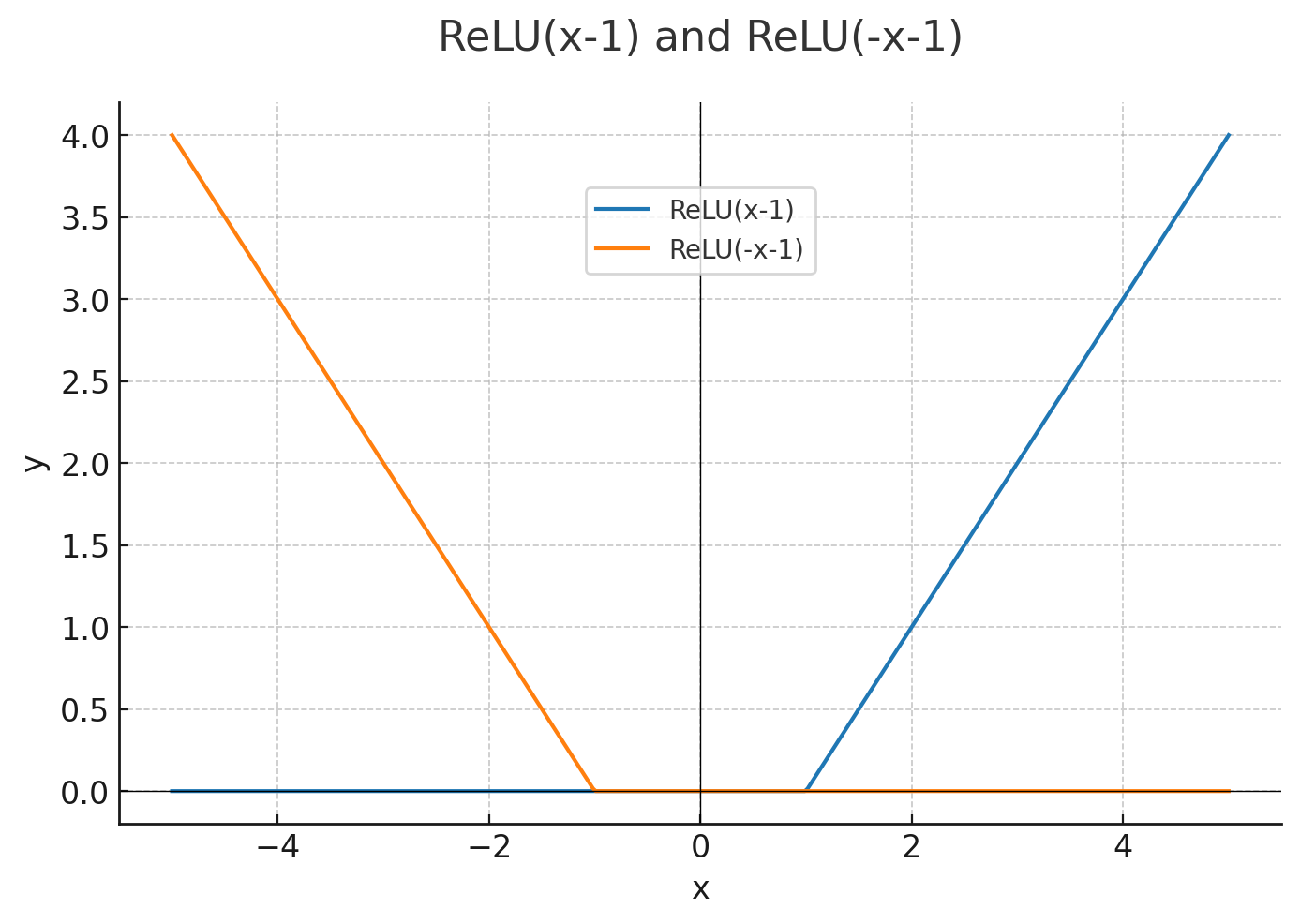}
        \caption{$\operatorname{ReLU}(x-1)$ and $\operatorname{ReLU}(-x-1)$}
        \label{fig:ReLU}
    \end{subfigure}
    \hfill
    \begin{subfigure}[b]{0.45\textwidth}
        \includegraphics[width=\textwidth]{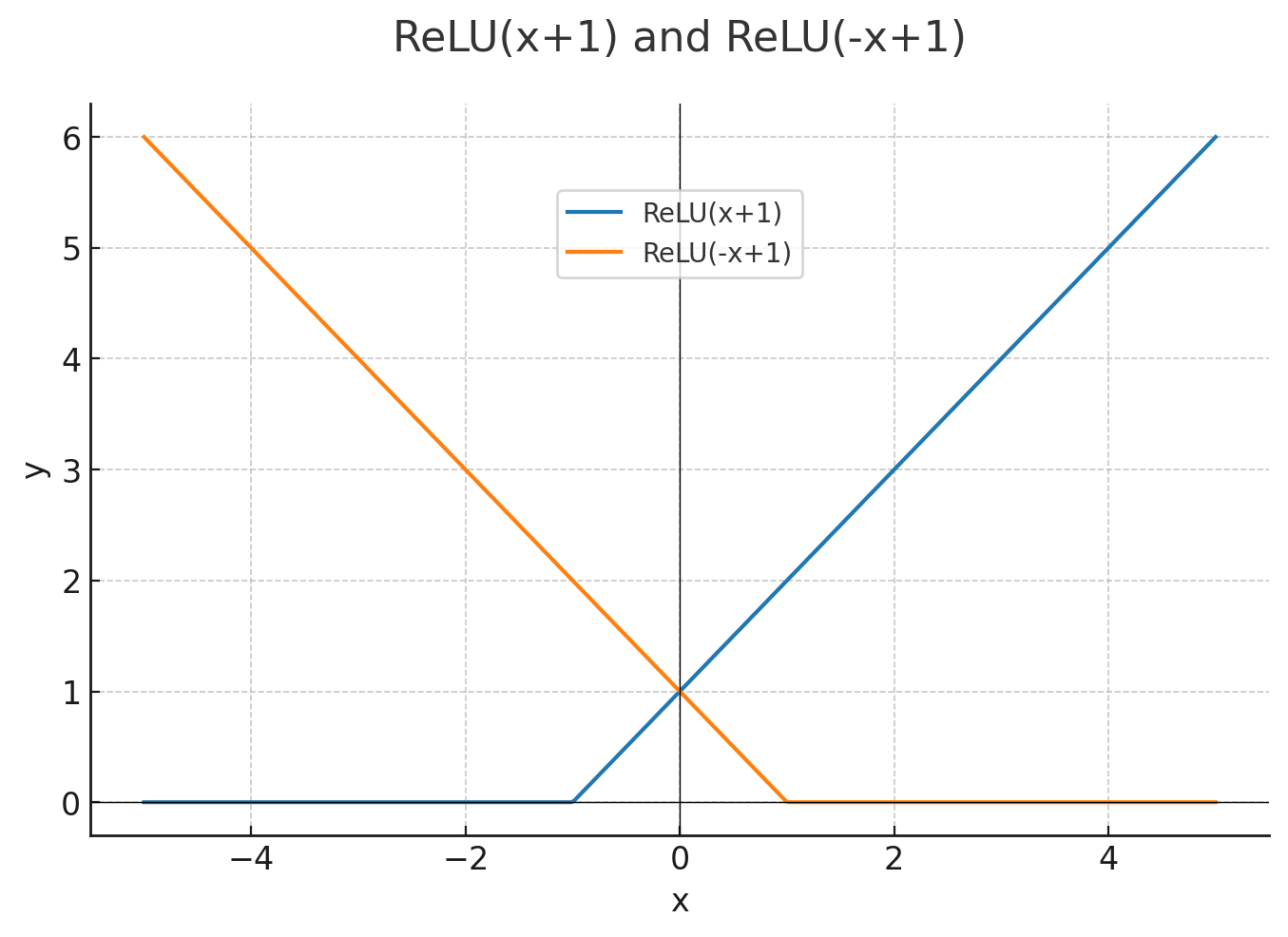}
        \caption{$\operatorname{ReLU}(x+1)$ and $\operatorname{ReLU}(-x+1)$}
        \label{fig:ReLUbad}
    \end{subfigure}
    \caption{ReLU-activated network with one-dimensional input under the reflection group.}
\end{figure}

When the supports are disjoint (Fig.~\ref{fig:ReLU}), group averaging significantly reduces the output value of the neural network. In contrast, when $b>0$, the two functions overlap (Fig.~\ref{fig:ReLUbad}), and the reduction achieved by group averaging is much weaker.
Hence, to make $\delta_{G,\Gamma,\sigma}$ small, the probability measure $\rho \in \Gamma$ should assign high probability to $b \leq 0$.

These ideas extend naturally to the multi-dimensional setting.
For instance, consider the two-dimensional domain $\Omega = [-1,1]^2$ with $G$ the group of $\pi/2$ rotations, i.e., the cyclic group $C_4 \subset SO(2)$.
The generator of $G$ is the $\pi/2$ rotation matrix
\[
   R_{\pi/2} \;=\;
   \begin{pmatrix}
      0 & -1 \\
      1 & \;\;0
   \end{pmatrix},
\]
so that $G = \{I, R_{\pi/2}, R_{\pi}, R_{3\pi/2}\}$. 

If $b \leq 0$, then after applying the group action, the overlap between $\sigma(\vw \cdot g_1 \vx + b)$ and $\sigma(\vw \cdot g_2 \vx + b)$ for $g_1\neq g_2 \in G$ can be very small.  
In particular, if the intersection points of the shifted hyperplane
\[
   \{\vx : \vw \cdot \vx + b = 0\}
\]
with the coordinate axes $x_1=0$ and $x_2=0$ lie entirely outside $\Omega$, i.e. if $-b\ge \|\vw\|_\infty$, then the supports of these activations can be disjoint under all group actions, resulting in zero overlap.  
For example, Fig.~\ref{fig:2d_route} illustrates the case $\vw=(1,1)$ and $b=-\tfrac{3}{2}$.
As a consequence, the group average significantly reduces the output magnitude, making $\delta_{G,\Gamma,\sigma}$ significantly smaller.

\begin{figure}[h!]
    \centering
    \begin{subfigure}[b]{0.45\textwidth}
        \includegraphics[width=\textwidth]{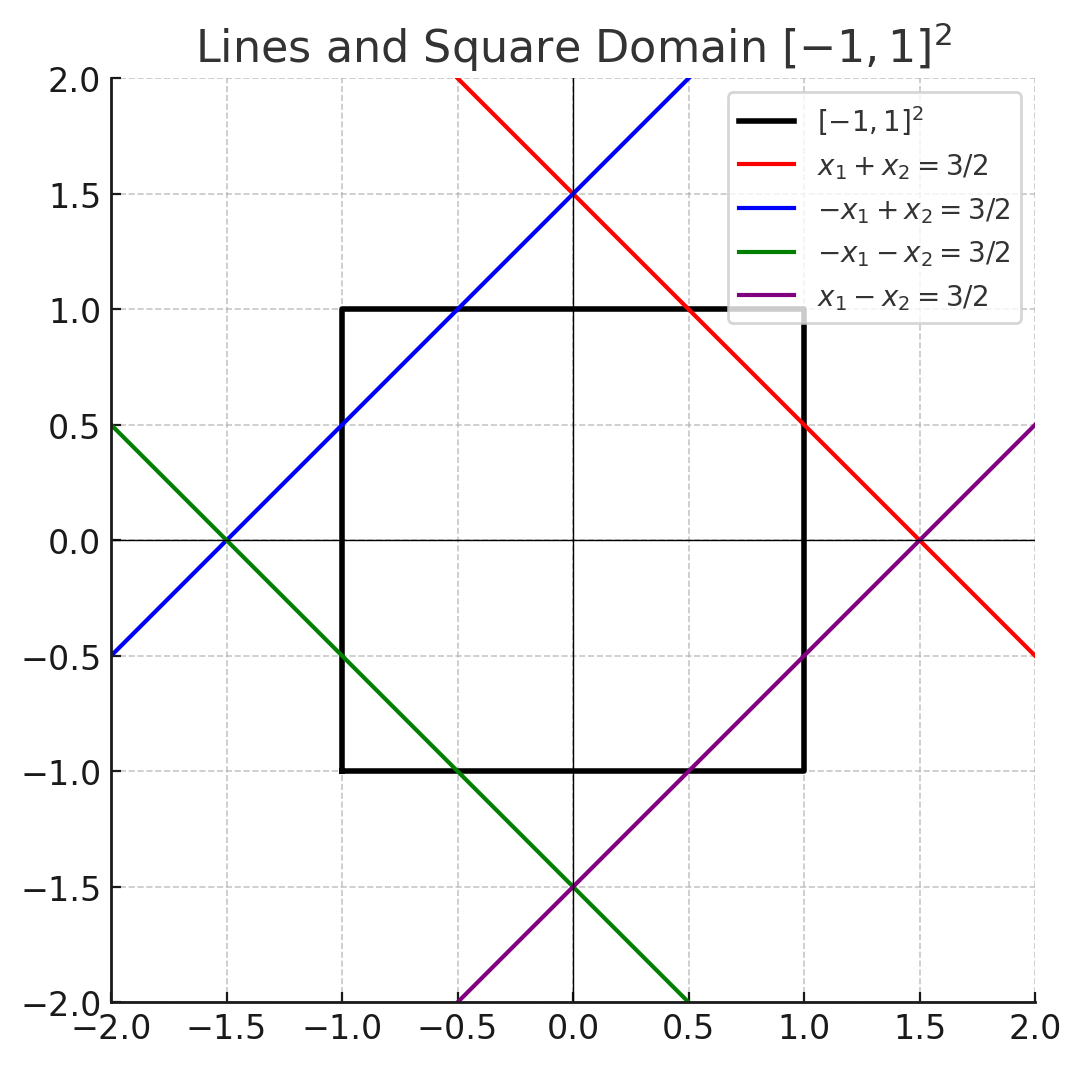}
        \caption{Case $\vw=(1,1)$ and $b=-\tfrac{3}{2}$ with $G = C_4 = \{I, R_{\pi/2}, R_{\pi}, R_{3\pi/2}\}$.}
        \label{fig:2d_route}
    \end{subfigure}
    \hfill
    \begin{subfigure}[b]{0.45\textwidth}
        \includegraphics[width=\textwidth]{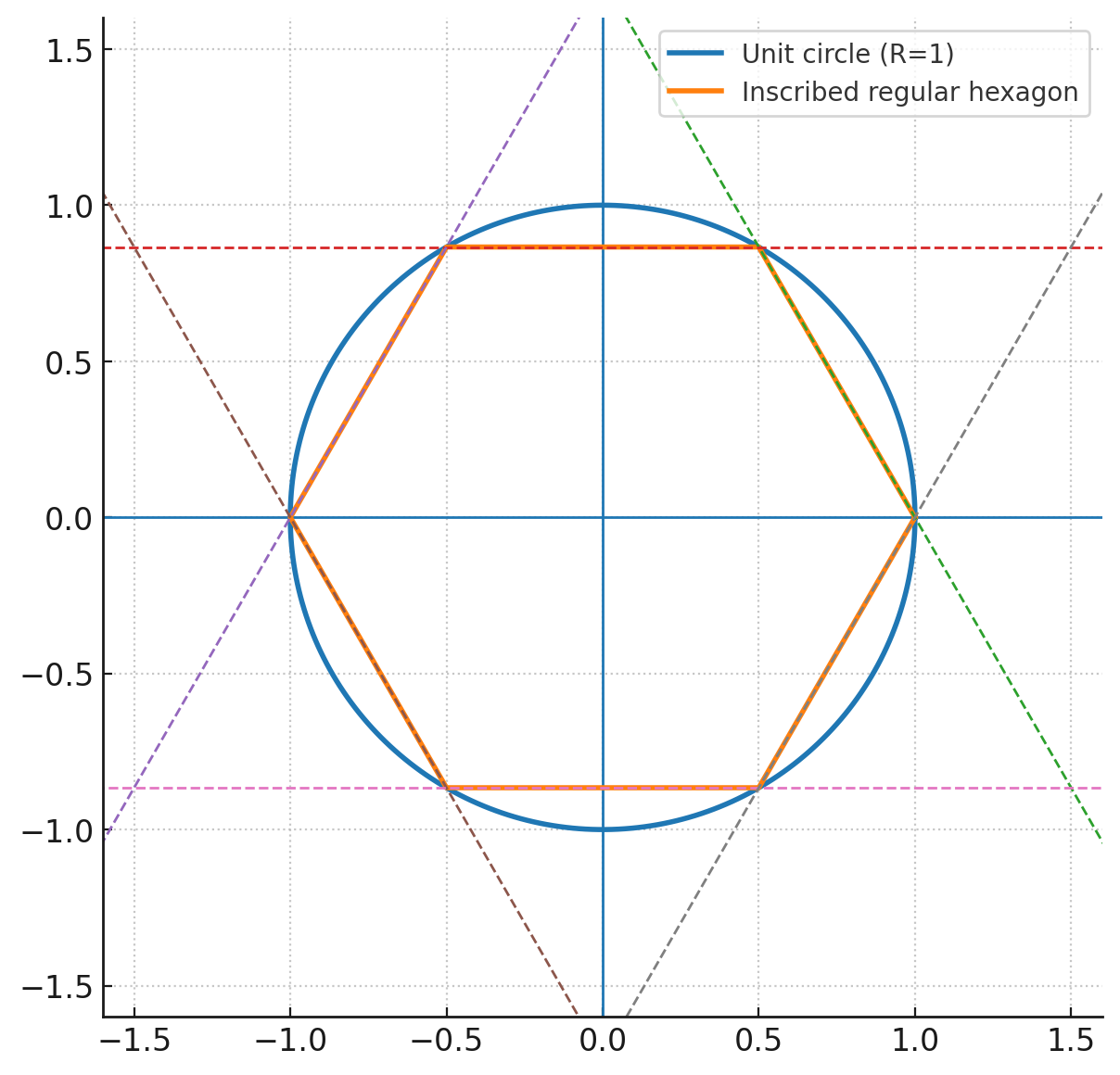}
        \caption{Disjoint-support case on the unit disk with $G= C_6\subset SO(2)$.}
        \label{fig:unit}
    \end{subfigure}
    \caption{ReLU-activated network with two-dimensional input under the discrete cyclic group $C_n\subset SO(2)$.}
\end{figure}

For larger cyclic groups $C_n \subset SO(2)$ with $n>4$, consider their action on the unit disk
\[
\Omega = \{\vx \in \mathbb{R}^2 : x_1^2 + x_2^2 \le 1\}.
\]  
Suppose $b < 0$. If the distance from the origin $(0,0)$ to the hyperplane ${\vx : \vw \cdot \vx + b = 0}$ exceeds the apothem of the inscribed regular $n$-gon (the distance from the origin to any of its sides), then
\[
\frac{-b}{\|\vw\|_2} \;\ge\; \cos\!\left(\frac{\pi}{n}\right).
\]  
In this case, the positive half-spaces ${\vx : \vw \cdot g\vx + b > 0}$ associated with different group elements $g \in C_n$ are disjoint. As a result, the overlaps between $\sigma(\vw \cdot g_1 \vx + b)$ and $\sigma(\vw \cdot g_2 \vx + b)$ for $g_1 \ne g_2$ are negligible or vanish completely, as illustrated in Fig.~\ref{fig:unit}, leading once again to a small $\delta_{G,\Gamma,\sigma}$ of order $\mathcal{O}(1/|G|)$.

\subsection{Compact Support Activation Functions}

Another class of activation functions that can yield a small $\delta_{G,\Gamma,\sigma}$ are those with compact support. As noted in Remark~\ref{rmk: act}, smooth compactly supported activations are included within our theoretical framework.

Consider the symmetric group $S_d$ in $d$ dimensions, which has $d!$ elements. For any $\tau \in S_d$, the group action is defined by
\[
   f(\tau \vx) := f\!\left(x_{\tau(1)}, \ldots, x_{\tau(d)}\right).
\]
Let $\Omega = [0,1]^d$ and partition it into $K^d$ subcubes for some $K \in \mathbb{N}$. Suppose $\sigma$ is an activation function with compact support, and assume that the support of $\sigma(\vw \cdot \vx + b)$ is contained in a single subcube,
\[
   \prod_{i=1}^d \left[ \tfrac{k_i}{K}, \tfrac{k_i+1}{K} \right],
   \quad \vk = (k_1, \ldots, k_d) \in \{0,1,\ldots,K-1\}^d.
\]

When averaging $\sigma(\vw \cdot \vx + b)$ over the group $S_d$, the compact support is distributed across
\[
   \frac{d!}{\prod_{i=1}^d i! \, s_i}
\]
distinct subcubes, where $s_i$ denotes the number of coordinates in $\vk$ that appear exactly $i$ times. If all components of $\vk$ are distinct, then each $\sigma(\vw \cdot \tau \vx + b)$ has disjoint support, and $\delta_{G,\Gamma,\sigma}\simeq 1/d!$. If each component of $\vk$ appears exactly twice (with $d$ even), then the compact supports overlap in pairs, leading to $\delta_{G,\Gamma,\sigma}\simeq 1/(d/2)!$. If all components of $\vk$ are identical, then the supports coincide completely under group averaging, giving $\delta_{G,\Gamma,\sigma}\simeq 1$.  

For illustration, consider the two-dimensional case shown in Fig.~\ref{fig:11}. If the compact support is located at $(0,1)$ or $(1,0)$, the symmetric action maps it to the other point, so the supports are disjoint and $\delta_{G,\Gamma,\sigma}\simeq 1/2$. If the compact support is located at $(0,0)$ or $(1,1)$, the symmetric action leaves it unchanged, so the supports coincide and $\delta_{G,\Gamma,\sigma}\simeq 1$.  

\begin{figure}[h!]
    \centering
    \includegraphics[width=0.5\linewidth]{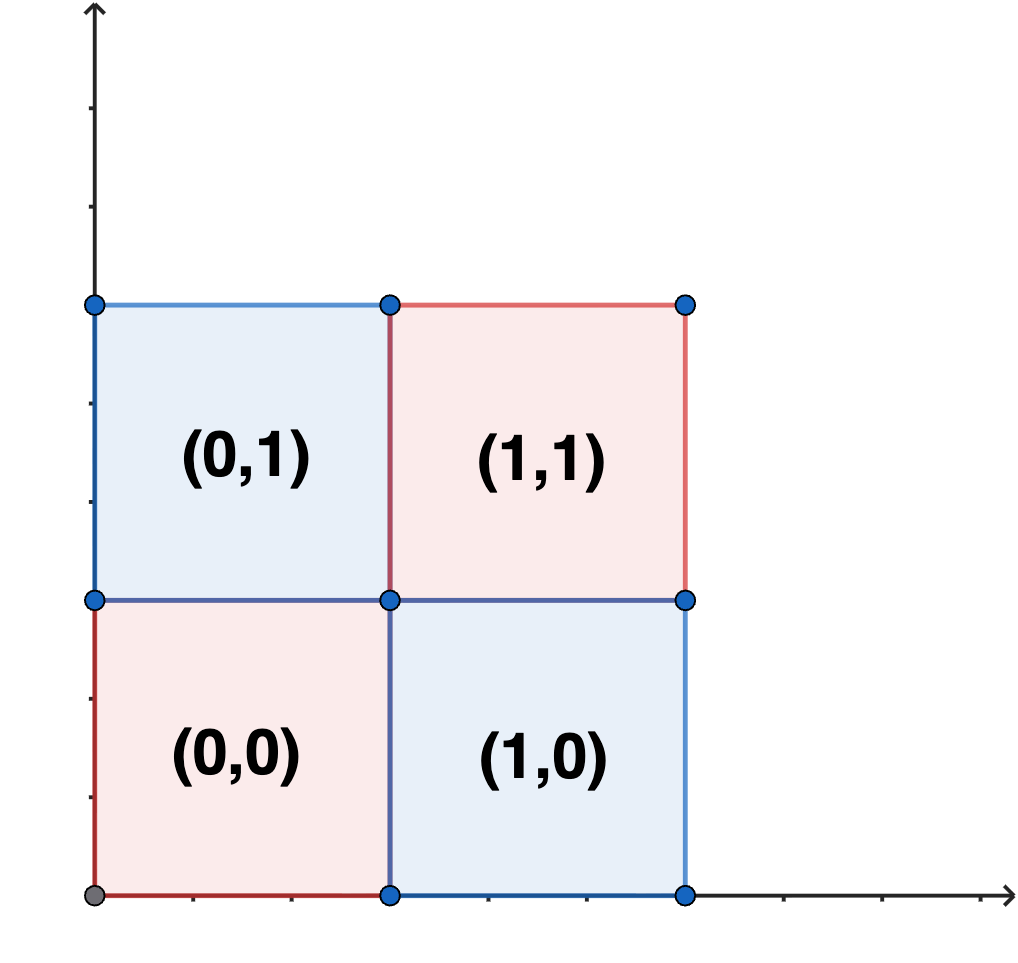}
    \caption{Two-dimensional domain $[0,1]^2$ with compact support activations under $S_2$.}
    \label{fig:11}
\end{figure}

In summary, if the distribution $\rho \in \Gamma$ places high probability on supports corresponding to $\vk$ with many distinct components, then $\delta_{G,\Gamma,\sigma}$ will be small. Conversely, if $\rho$ favors supports with repeated components, then $\delta_{G,\Gamma,\sigma}$ will be close to one. Many other cases can be analyzed, but here we provide two representative examples to illustrate when $\delta_{G,\Gamma,\sigma}$ is small and when it is close to one, depending on the group and the activation function. 

This discussion can also be generalized to the setting of $G$–equivariant functions, such as antisymmetric functions. In that case, the range of possible values for $\delta_{G,\Gamma,\sigma}$ can be much broader, since even when compact sets intersect, sign changes may cancel certain positive contributions to the intersection, thereby altering the value of $\delta_{G,\Gamma,\sigma}$. A detailed analysis of this phenomenon for $G$-equivariant functions is beyond the scope of the present paper and will be left to future work.

\section*{Acknowledgments}
This material is based upon work supported by the U.S. National Science Foundation under the award DMS-2502900 and by the Air Force Office of Scientific Research (AFOSR) under Grant No.~FA9550-25-1-0079.

\bibliographystyle{plain}
\bibliography{references}

\end{document}